\documentclass[11pt,a4paper]{article}
\usepackage[T1]{fontenc}
\usepackage[utf8]{inputenc}
\usepackage{authblk}

\usepackage{times}
\usepackage{soul}
\usepackage{url}
\usepackage[hidelinks]{hyperref}
\usepackage[utf8]{inputenc}
\usepackage[small]{caption}
\usepackage{graphicx}
\usepackage{amsmath}
\usepackage{amsthm}
\usepackage{booktabs}
\usepackage{algorithm}
\usepackage{algorithmic}
\usepackage{amssymb, amsfonts}
\usepackage{subfigure}

\newtheorem{theorem}{Theorem}
\newtheorem{lemma}[theorem]{Lemma}

\newtheorem{assumption}{Assumption}

\title{Intermittent Pulling with Local Compensation for Communication-Efficient Federated Learning}

\author[1,2]{Haozhao Wang}
\author[2,3]{Zhihao Qu}
\author[2]{Song Guo}
\author[1]{Xin Gao}
\author[1]{Ruixuan Li}
\author[4]{Baoliu Ye}
\affil[1]{School of Computer Science and Technology, \authorcr \textit{Huazhong University of Science and Technology}}
\affil[2]{Department of Computing, The Hong Kong Polytechnic University}
\affil[3]{College of Computer and Information, Hohai University}
\affil[4]{Department of Computer Science and Engineering, Nanjing University \authorcr \{hz\_wang, emgaox, rxli\}@hust.edu.cn, song.guo@polyu.edu.hk, quzhihao@hhu.edu.cn, yebl@nju.edu.cn}
\date{} 

\begin{document}

\maketitle
\begin{abstract}
    Federated Learning is a powerful machine learning paradigm to cooperatively train a global model with highly distributed data. A major bottleneck on the performance of distributed Stochastic Gradient Descent (SGD) algorithm for large-scale Federated Learning is the communication overhead on pushing local gradients and pulling global model. In this paper, to reduce the communication complexity of Federated Learning, a novel approach named Pulling Reduction with Local Compensation (PRLC) is proposed. Specifically, each training node intermittently pulls the global model from the server in SGD iterations, resulting in that it is sometimes unsynchronized with the server. In such a case, it will use its local update to compensate the gap between the local model and the global model. Our rigorous theoretical analysis of PRLC achieves two important findings. First, we prove that the convergence rate of PRLC preserves the same order as the classical synchronous SGD for both strongly-convex and non-convex cases with good scalability due to the linear speedup with respect to the number of training nodes. Second, we show that PRLC admits lower pulling frequency than the existing pulling reduction method without local compensation. We also conduct extensive experiments on various machine learning models to validate our theoretical results. Experimental results show that our approach achieves a significant pulling reduction over the state-of-the-art methods, e.g., PRLC requiring only half of the pulling operations of LAG. 
\end{abstract}

\section{Introduction}
The explosion of data and rapid increase in model size have led to great attention to distributed 
machine learning approaches. Recently, Federated Learning has been proposed to enable a large number of 
workers, e.g., phones, tablets, and sensors, to cooperatively train a global model without exposing 
their own data. The most widely used training algorithm is distributed Stochastic Gradient Descent (SGD), 
where individual workers iteratively refresh the global model located on the server via operations of 
pushing gradients to and pulling model from the server. The huge communication overhead imposes a bottleneck on the performance of Federated Learning.

Recently the mainstream to reduce the communication overhead is in two orthogonal lines: compressing the transmission data and reducing the pushing/pulling operations. Although data compression based methods, e.g., sparsification~\cite{DBLP:conf/aaai/ChenCBAZG18,shi2019convergence} 
and quantization~\cite{alistarh2017qsgd,wen2017terngrad}, have shown effectiveness in reducing transmission bits, the cost of other communication overhead,  
e.g., searching servers, queuing and propagating messages in the network, could not be ignored~\cite{wang2019scalable}. Moreover, the transmission energy of low-power wireless devices is often dominated by activation of communication modules and less so by the actual signal amplitude or feature output dimension~\cite{zhu2019cost}. 
To this end, it is of great significance to investigate the approach of reducing the number of pushing/pulling operations.

To reduce the pushing operations, Wang et al.~\cite{luping2019cmfl} propose excluding the workers with outlier updates deviated from the average update of all workers.
However, their method could not be used to reduce the pulling operations because the global model is unique and has no outliers. 
Recently, Chen et al. propose an adaptive algorithm LAG to reduce the pulling operations, in which the outdated model and gradient are reused and thus the workers do not need to pull the global model in some specific iterations~\cite{chen2018lag}. However, LAG requires the gradient varying slowly. This assumption holds only when the Batch Gradient Descent based algorithms. Our experiment shows that it indeed does not help much for SGD.

In this paper, we propose a novel method to reduce the communication overhead of SGD-based Federated Learning named Pulling Reduction with Local Compensation (PRLC). The main idea is that workers intermittently pull global model from the server and when they decide not to pull, they compensate the gap with the local updates. Specifically, the local update is utilized to approximate the average update of all workers which is used to compensate the gap between local model and global model. We prove that our method yields smaller pulling ratio (proportion of workers that synchronizes with the global model) than the pulling rounds reduction method without local compensation. Furthermore, our method largely improves the training speed. We also theoretically show that our method has better scalability than the popular method Asynchronous SGD (ASGD)~\cite{dean2012large,ho2013more}, and experimentally present that our method outperforms ASGD in terms of the convergence time.

Our contributions are summarized as follows:
\begin{itemize}
    \item We propose a novel PRLC method that only a portion of workers pull global model from servers in each iteration and those not pulling use their local updates to compensate the gap.
    \item To the best of our knowledge, we are the first analyzing the convergence rate of distributed SGD with pulling reduction and local compensation in both strongly-convex and non-convex cases. The theoretical results show that the asymptotic convergence rate of our method is in the same order as the non-compression method.
    \item We provide rigorous analysis to show the advantage of local compensation in PRLC. It yields a lower pulling ratio than the reduction method without local compensation. In addition, PRLC has better scalability than ASGD due to the linear speedup with respect to the number of training nodes.
    \item We conduct extensive experiments in various machine learning tasks including both convex and non-convex models. The results show that our proposed method achieves significant improvement.
\end{itemize}

This paper is organized as follows. Section~\ref{Section:RelatedWork} presents the related work. Then, the preliminaries about Federated Learning are introduced in Section~\ref{Section:Preliminary}. After that, we propose PRLC and analyze its convergence rate in Section~\ref{Section:MethodologyAndAnalysis}. In Section~\ref{Section:Evaluations}, experiments are performed to show the efficiency of our method. Finally, the conclusions are drawn in Section~\ref{Section:Conclusion}.

\section{Related Work}\label{Section:RelatedWork}
Federated Learning can be viewed as a special case of distributed machine learning with guaranteed data privacy of each worker. Many works have been proposed to improve its efficiency. Federated Learning can be viewed as a special case of distributed machine learning with protecting the data privacy of each worker. Many works have been proposed to improve its efficiency. 

To reduce the size of communication data, compression-based methods, e.g., gradient sparsification and quantization, have been proposed recently. Gradient sparsification~\cite{Yujun,alistarh2018convergence} reduces the communication cost by only transmitting a portion of the dimensions of the gradient. AdaComp~\cite{DBLP:conf/aaai/ChenCBAZG18} proposes a dynamic strategy for the selection of dimensions. Xiao et al.~\cite{xiao2017fast} sample dimensions of the gradient to accelerate the training process. Gradient quantization~\cite{DBLP:journals/corr/ZhouNZWWZ16,wen2017terngrad} quantizes the value of gradient from 32-bit float number to some lower bit representation with lower precision. QSGD~\cite{alistarh2017qsgd} is a general framework for quantization, in which the relationship between the quantized level and the convergence rate has been established. 

The asynchronous distributed learning algorithm has been proposed to improve the computation efficiency over the synchronous method, of which the key concept is to improve the hardware computing efficiency by sacrificing some convergence rate~\cite{dean2012large,gu2018accelerated}. Another way is to overlap the computation and communication to reduce the run time~\cite{wang2019scalable,li2018pipe}. Shen et al.~\cite{shen2019faster} further improve the computation efficiency by overlapping the communication with multiple local computation steps. Though these methods have made great progress in reducing the training time, communication overhead is not alleviated.

There are also some methods that improe convergence efficiency with less iterations. These methods could also reduce the communication overhead as by-product. Newton method~\cite{shamir2014communication,wang2018giant} 
uses the second-order information to speed up training and thus incurs less the communication rounds. 
However, these methods are easily trapped to the saddle point \cite{ge2015escaping,jin2017escape}. Alternatively, some other methods use large batch size with a large learning rate~\cite{goyal2017accurate,you2018imagenet} to accelerate convergence, 
but they do not have theoretical guarantees and may result in accuracy loss~\cite{yin2018gradient}. 

Recently, the pushing/pulling operations reduction methods are proposed to specifically reduce the communication overhead with the convergence preserved. Wang et al.~\cite{luping2019cmfl} 
propose reducing the number of pushing communication operations by excluding the updates of some workers, but their method could not be applied to reducing the number of pulling communication operations. Recently LAG \cite{chen2018lag} and its variant LAQ~\cite{sun2019communication} are recently proposed to reduce the number of pulling communication operations by adaptively reusing the outdated gradients. However, reusing outdated gradients only holds in BGD of which the gradients vary little. Consequently, there is still a deficiency for pulling operations reduction for SGD which is solved by our work.

\section{Preliminaries}\label{Section:Preliminary}
In this paper, we seek to solve the sum optimization problem which is general in the machine learning field:
\begin{equation}
    \label{eq:Objective}
    \mathop{argmin}\limits_{\omega} F(\omega) = \frac{1}{N} \sum\limits_{i=1}^{N}f(\omega).
\end{equation}
where $\omega \in \mathcal{R}^d$ is the model and $N$ is the number of samples.

Throughout this paper, we use $\|\omega\|$ to denote the $L_2$ norm of vector $\omega$, $w^*$ to denote the optimal parameter, $\nabla F(\omega)$ to denote the full gradient with respect to $\omega$, and $g(\omega; \xi )$ to denote one stochastic gradient with respect to a mini-batch $\xi$. 

We make the following assumptions to the objective function. These assumptions are commonly used in 
stochastic optimization in both convex and non-convex cases
\cite{DBLP:journals/siamrev/BottouCN18,shi2019flexible,lian2015asynchronous}.

\begin{assumption}
\emph{($L$-smooth function)} The objective function $F$ is $L$-smooth with Lipschitz constant $L>0$, $\|\nabla F(\omega_1)-\nabla F(\omega_2)\| \leq L\|\omega_1-\omega_2\|$, which indicates that
\begin{equation*}
F(\omega_{2}) - F(\omega_1) \le \nabla F(\omega_1)(\omega_{2}-\omega_1)^T + \frac{L}{2}\|\omega_{2}-\omega_1\|^2
\end{equation*}
\end{assumption}

\begin{assumption}
(\emph{bounded value}) In the sequence of iterations, the global parameters $\omega_1, \omega_2, \cdots$ are contained in an open set over which $F$ is bounded below by a scalar $F^*$, i.e., $\forall t, F^* \le F(\omega_t)$.
\end{assumption}

\begin{assumption}
(\emph{bounded gradient}) The $L_2$ norm of stochastic gradient is bounded by a constant G, i.e., $\| g(\omega; \xi) \| \le G$.
\end{assumption}

\begin{assumption}
(\emph{unbiased gradient}) The stochastic gradient is unbiased for any parameter $\omega$, i.e.,
\begin{equation*}
\mathbb{E}_{\xi}[g(\omega; \xi)]=\nabla F(\omega)
\end{equation*}
\end{assumption}

\begin{assumption}
(\emph{bounded variance}) The variance of stochastic gradient is bound by a constant $\sigma^{2}$, i.e.,
\begin{equation*}
\mathbb{E}_{\xi}\left[\|g(\omega; \xi)-\nabla F(\omega)\|^{2}\right] \leq \sigma^{2}
\end{equation*}
\end{assumption}

\section{Method and Analysis}\label{Section:MethodologyAndAnalysis}

The workflow of PRLC is shown in Algorithm \ref{alg:algorithm1}. Line 5 presents the reduction of the number of communication rounds for each worker. Specifically, each worker $i$ pulls the global model with a probability $r$ and the workers that not pull the global model update their local model with the local gradient. The intuition lies in using the local gradient to approximate the average of all gradients to compensate the gap between the local model and the global model. We in Section~\ref{section:PR} theoretically show that the method with local compensation achieves a better tolerance of the low pulling ratio than the method without local compensation. 

\begin{algorithm}[htbp]
    \caption{Distributed SGD with PRLC}
    \label{alg:algorithm1}
    \begin{algorithmic}[1]
    \STATE {\bfseries Input:} Initialize $\omega_1^i=\omega_1$, learning rate $\eta_0$, 
    pulling ratio $r$, and iterations $T$
    \FOR{$t=1$ {\bfseries to} $T$}
        \STATE Each worker $i$ computes $g(\omega_t^i; \xi_t^i)$ in parallel;
        \STATE $\omega_{t+1} = \omega_t - \frac{\eta_t}{P}\sum\limits_{i=1}^{P}g(\omega_t^i; \xi_t^i)$;
        \STATE Each worker $i$ updates its local model with the pulled global model 
                or its local gradient:
               \begin{equation*}
                        \omega_{t+1}^i =\left\{
                        \begin{array}{rcl}
                        \omega_{t+1},                                       && {with\ probability\ r,} \\
                        \omega_{t}^i - \eta_t g(\omega_t^i; \xi_t^i),       && {otherwise.}     \\
                        \end{array} \right.
                    \end{equation*}
    \ENDFOR
    \end{algorithmic}
\end{algorithm}

\subsection{Analysis of PRLC}
In this section, we analyze the convergence rate of RPLC. The theoretical results show that the PRLC essentially admits the same convergence rate as classical synchronous SGD. Due to the page limitation, we give the sketch of the proof for the main theorem and omit other proofs. Details are presented in full paper.

Different from synchronized SGD mechanisms, in PRLC the global parameter and local parameter are not exactly the same for each worker, since each worker randomly pulls the global parameter and updates its own model. For any worker $i$ and iteration $t$, if worker $i$ last updated the global parameter in iteration $t-k$, where $k = 0,1,\ldots, t-1$, i.e., $\omega_{t-k}^i = \omega_{t-k}$ then we have:
\begin{equation}
\label{equ:localP}
\omega_{t}^i = \omega_{t-k} - \sum_{j=1}^k \eta g(\omega_{t-j}^i; \xi_{t-j}^i)
\end{equation}
and
\begin{equation}
\label{equ:globalP}
\omega_t = \omega_{t-k} - \frac{\eta}{p}\sum_{j=1}^k \sum_{i=1}^P \eta g(\omega_{t-j}^i; \xi_{t-j}^i)
\end{equation}
Therefore, we derive the bound of the differential between global parameter and local parameter and have that:
\begin{myfont}
\begin{align}
\label{equ:parameterGap}
 \| \omega_t - \omega_t^i \|^2 
 & =  \| \sum_{j=1}^k \eta g(\omega_{t-j}^i; \xi_{t-j}^i) - \frac{\eta}{P}\sum_{j=1}^k \sum_{i=1}^P \eta g(\omega_{t-j}^i; \xi_{t-j}^i) \|^2 \nonumber\\
 &\le  2 \eta^2 \|\sum_{j=1}^k g(\omega_{t-j}^i; \xi_{t-j}^i)\|^2 + \frac{2\eta^2}{P^2} \|\sum_{j=1}^k \sum_{i = 1}^P g(\omega_{t-j}^i; \xi_{t-j}^i)\|^2
\end{align}
\end{myfont}

Now, we are ready to derive the convergence property of our proposed algorithm.
\begin{theorem}\label{Theorem:T1}
(\textbf{Convergence property, non-convex objective})When PRLC is running with a fixed learning rate $\eta$ for all iterations and the learning rate for every iteration satisfies:
\begin{equation}
\label{equ:eta}
0 < \eta \le \frac{-2 L r^2 + \sqrt{4L^2r^4 + 32L^2r^2(1-r)(2-r)}}{16L^2(1-r)(2-r)},
\end{equation}
then the expected average squared gradient norms of $F$ are bounded for all $T\in \mathbb{N}$:
\begin{myfont}
\begin{equation}
    \label{theorem1:bound1}
\frac{1}{T}\sum_{t=1}^T \mathbb{E} \|\nabla F(\omega_t)\|^2 \le \frac{2|F(\omega_{1}) - F(\omega^*)|}{\eta T} + 2A_{\eta,L,r,G,\sigma}
\end{equation}
\end{myfont}
where $A_{\eta,L,r,G,\sigma,P} = \frac{2\eta^2 L^2 (1-r)(2-r)(PG^2 + 2\sigma^2) + L \eta \sigma^2 r^2}{Pr^2} $.
\end{theorem}
\begin{proof}
We give the sketch of the proof.

Let $\xi$ be the set of mini-batches. For any worker $i$ and iteration $t$, by taking the expectation with respect to $\xi$ and $k$, we can derive:
\begin{equation}
\label{equ:expectation1}
\mathbb{E}_{\{k,\xi\}} \|\sum_{j=1}^k g(\omega_{t-j}^i; \xi_{t-j}^i)\|^2 \le \frac{(1-r)(2-r)G^2}{r^2}
\end{equation}
\begin{align}
\label{equ:expectation2}
{} & \mathbb{E}_{\{k,\xi\}} \|\sum_{j=1}^k \sum_{i = 1}^P g(\omega_{t-j}^i; \xi_{t-j}^i)\|^2 \nonumber\\
= & r \sum_{\ell = 0}^{t-1} \sum_{m=1}^\ell [(1-r)^{t-m}(t-m) \mathbb{E}_{\xi} \| \sum_{i=1}^P g(\omega_\ell^i; \xi_\ell^i) \|^2 ]
\end{align} 
Combining with (\ref{equ:parameterGap}) and $L$-smooth assumption, we have:
\begin{myfont}
\begin{align}
\label{equ:gradientGap}
{} & \mathbb{E}_\xi \left[ \| \nabla F(\omega_t) - \nabla F(\omega_t^i) \|^2 \right] \le \frac{4\eta^2 L^2 G^2 (1-r)(2-r)}{r^2} \nonumber\\
+ & \frac{4\eta^2 L^2 r}{P^2} \sum_{\ell = 0}^{t-1} \sum_{m=1}^\ell [(1-r)^{t-m}(t-m) \mathbb{E}_\xi \| \sum_{i=1}^P g(\omega_\ell^i; \xi_\ell^i) \|^2 ] 
\end{align}
\end{myfont}
Note that the inequality $\mathbb{E}_{\xi} [\| \sum_{i = 1}^P g(\omega_t^i; \xi_t^i) \|^2] \le 2P\sigma^2 + 2\sum_{i=1}^P \| \nabla F(\omega_t^i)  \| ^2$ holds for any $\omega_t^i$ and $\xi_t^i$, we can derive that:
\begin{align}
{} & \mathbb{E}_{\xi} [ F(\omega_{t+1}) - F(\omega_t)] \nonumber\\
\le & \mathbb{E}_{\xi} [\nabla F(\omega_t)(\omega_{t+1}-\omega_t)^T] + \frac{L}{2}\mathbb{E}_{\xi}\|\omega_{t+1}-\omega_t\|^2 \nonumber\\
= & \mathbb{E}_{\xi} [\nabla F(\omega_t)(-\frac{\eta}{P}\sum_{i=1}^P g(\omega_t^i;\xi_t^i))^T ] + \frac{L}{2} \mathbb{E}_{\xi} \|\frac{\eta}{P}\sum_{i=1}^P g(\omega_t^i;\xi_t^i)\|^2 \nonumber\\
= & -\frac{\eta}{2}\mathbb{E}_{\xi}\|\nabla F(\omega_t)\|^2 + \frac{2 L\eta^2 - \eta}{2P^2}\mathbb{E}_{\xi}\|\sum_{i=1}^P \nabla F(\omega_t^i)\|^2  \nonumber\\
+ &  \frac{L\eta^2\sigma^2}{P} + \frac{\eta}{2P}\sum_{i=1}^P \mathbb{E}_{\xi}\|  \nabla F(\omega_t) - \nabla F(\omega_t^i)\|^2
\end{align}
Summing up the above equation for $t=1$ to $T$ for both sides and replacing $\mathbb{E}_\xi \left[ \| \nabla F(\omega_t) - \nabla F(\omega_t^i) \|^2 \right]$ according to (\ref{equ:gradientGap}) immediately yield that: 
\begin{align}
{} & \mathbb{E} F(\omega_{t+1}) - F(\omega_1) \nonumber\\
\le & -\frac{\eta}{2}\sum_{t=1}^T \mathbb{E} \|\nabla F(\omega_t)\|^2 + \frac{2 L\eta^2 - \eta}{2P^2} \mathbb{E} \sum_{t=1}^T \|\sum_{i=1}^P \nabla F(\omega_t^i)\|^2 \nonumber\\
+ & \frac{2\eta^3 L^2 G^2 T (1-r)(2-r)}{r^2} + \frac{L\eta^2\sigma^2T}{P} \nonumber\\
+ & \frac{2\eta^3 L^2 (1-r)(2-r)}{P^2 r^2} \sum_{t=1}^{T-1} \mathbb{E} \| \sum_{i=1}^P g(\omega_\ell^i; \xi_\ell^i) \|^2
\end{align}
By replacing $\mathbb{E} \| \sum_{i=1}^P g(\omega_\ell^i; \xi_\ell^i) \|^2$ according to the result of lemma 1, we get that
\begin{align}
\mathbb{E} F(\omega_{t+1}) - F(\omega_1) & \le  -\frac{\eta}{2}\sum_{t=1}^T \|\nabla F(\omega_t)\|^2 + A_{\eta,L,r,G,\sigma} T \eta  \nonumber\\
{} & +  B_{\eta,L,r,P} \mathbb{E} \sum_{t=1}^T \|\sum_{i=1}^P \nabla F(\omega_t^i)\|^2 
\end{align}
where $A_{\eta,L,r,G,\sigma,P} = \frac{2\eta^2 L^2 (1-r)(2-r)(PG^2 + 2\sigma^2) + L \eta \sigma^2 r^2}{r^2P} $ and \\
$B_{\eta,L,r,P} = \frac{8\eta^3 L^2 (1-r)(2-r) + (2L\eta^2 - \eta) r^2}{2P^2 r^2}$ are two constants.

By setting $B_{\eta,L,r,P} < 0$, we can derive the satisfied stepsize, i.e., $0 < \eta \le \frac{-2 L r^2 + \sqrt{4L^2r^4 + 32L^2r^2(1-r)(2-r)}}{16L^2(1-r)(2-r)}$. Then we have
\begin{equation*}
\mathbb{E} F(\omega_{t+1}) - F(\omega_1) \le  -\frac{\eta}{2}\sum_{t=1}^T \|\nabla F(\omega_t)\|^2 + A_{\eta,L,r,G,\sigma,P} T \eta
\end{equation*}

Since $F(\omega_{T+1})$ is bound by $F(\omega^*)$, we have
\begin{equation}
\label{equ:convergence}
\frac{1}{T}\sum_{t=1}^T \mathbb{E} \|\nabla F(\omega_t)\|^2 \le \frac{2|F(\omega_{1}) - F(\omega^*)|}{\eta T} + 2A_{\eta,L,r,G,\sigma,P}
\end{equation}
\end{proof}

According to (\ref{equ:convergence}), the average norm of gradient converges to a non-zero constant $2A_{\eta,L,r,G,\sigma,P}$ under a fixed learning rate with $T\rightarrow \infty$. 
\begin{theorem}
Let $ \eta = \sqrt{\frac{[F\left(\omega_{1}\right)-F\left(\omega^*\right)]P }{L \sigma^{2} T}}$
then for any iteration times
\begin{myfont}
\begin{equation}
T \ge \frac{256[F\left(\omega_{1}\right)-F\left(\omega^*\right)]L^3(1-r)^2(2-r)^2P}{\sigma^2 [-2 L r^2 + \sqrt{4L^2r^4 + 32L^2r^2(1-r)(2-r)}]^2}
\end{equation}
\end{myfont}
Algorithm 1 satisfy the following ergodic convergence rate
\begin{equation}
\label{equ:iteration}
\frac{1}{T} \sum_{t=1}^{T} \mathbb{E}\left\|\nabla F\left(\omega_{t}\right)\right\|^{2} \preceq O( \frac{1}{\sqrt{PT}}) + O(\frac{P}{T}),
\end{equation}
where $\preceq$ denotes order inequality, which means less than or equal to up to a constant factor.
\end{theorem}
\begin{proof}
    See Appendix of full paper.
\end{proof}

\textbf{Discussion.} Clearly, $O(\frac{1}{\sqrt{PT}})$ dominates the convergence as $T$ is large. By diminishing $\eta$ during the learning process, i.e., $\eta = O(\sqrt{P/T})$, it is easy to find that PRLC has a convergence rate of $O(1/\sqrt{PT})$, as shown in the following theorem. This result suggests that PRLC essentially admits the same convergence rate as non-compression distributed SGD since it has the asymptotical convergence rate $O(1/\sqrt{PT})$, which means it has the linear speedup property and high efficiency in large-scale distributed learning.

\subsection{PRLC Tolerates Lower Ratio than PR}\label{section:PR}
The algorithm of Pulling communication Reduction (PR) without local compensation is represented in Algorithm~\ref{alg:algorithm2}. The difference lies in Line~5, in which the workers of PR 
not pulling the global model do not compensate the gap with the local update. Its convergence is shown in the following Theorem~\ref{Theorem:PR}. 

\begin{theorem}\label{Theorem:PR}
    (Convergence property, non-convex objective)When algorithm is running with fixed learning rate $\eta$ for all iterations, and the learning rate for every iteration satisfies:
    \begin{equation}
    \label{equ:eta}
    0 < \eta \le \frac{-4 L + \sqrt{16L^2 + 32L(1-r)}}{16L(1-r)},
    \end{equation}
    and large ratio $0.5 \le r$, then the expected average squared gradient norms of $F$ are bounded for all $T\in \mathbb{N}$:
    \begin{align}
        \label{theoremPR:bound1}
        &\frac{1}{T}\sum_{t=1}^T \mathbb{E} \|\nabla F(\omega_t)\|^2 \le \frac{2(F(\omega_{1}) - F(\omega^*)}{\eta T} +  \frac{2L\eta\sigma^2}{P}  \nonumber \\
        &+ \frac{4\eta^2 L \sigma^2}{P}\left[\frac{1-r}{2r-1} - \frac{(1-r)(1-2^T(1-r)^T)}{(2r-1)^2T} \right]
    \end{align}
    \end{theorem}
\begin{proof}
    See Appendix of full paper.
\end{proof}

Compared to PR, PRLC converges with more relaxed constraints for the value of pulling ratio. 
Consequently, PRLC tolerates lower ratio than PR theoratically.

\begin{algorithm}[htbp]
    \caption{Distributed SGD with PR}
    \label{alg:algorithm2}
    \begin{algorithmic}[1]
    \STATE {\bfseries Input:} Initialize $\omega_1^i=\omega_1$, learning rate $\eta_0$, 
    pulling ratio $r$, and iterations $T$
    \FOR{$t=1$ {\bfseries to} $T$}
        \STATE Each worker $i$ computes $g(\omega_t^i; \xi_t^i)$ in parallel;
        \STATE $\omega_{t+1} = \omega_t - \frac{\eta_t}{P}\sum\limits_{i=1}^{P}g(\omega_t^i; \xi_t^i)$;
        \STATE Each worker $i$ updates its local model with global model or 
               local compensation:
               \begin{equation*}
                        \boldsymbol{\omega}_{t+1}^i =\left\{
                        \begin{array}{rcl}
                        \omega_{t+1},                                       && {with\ probability\ r,} \\
                        \omega_{t}^i,       && {otherwise.}     \\
                        \end{array} \right.
                    \end{equation*}
    \ENDFOR
    \end{algorithmic}
\end{algorithm}

\subsection{PRLC Scales Better than ASGD}
\begin{theorem}
    PRLC has better scalability than ASGD.
\end{theorem}
\begin{proof}
The convergence rate of ASGD is significantly limited by the staleness 
and the limitation becomes even serious as the system scales. 
We here consider the best case of ASGD where the maximum staleness is minimized, 
i.e., being equivalent to the number of workers $P$. Derived from Theorem~1 in 
Lian~et.al.\cite{lian2015asynchronous}, the convergence rate of ASGD with 
a fixed learning rate is 
\begin{equation}
    \label{eq:ASGD}
    \frac{1}{T} \sum\limits_{t=1}^{T} \mathbb{E} \|\nabla F(\omega_t)\|_2^2 
    \leq \frac{2(F(\omega_1) - F(\omega^*))}{T\eta} + \eta(C_0 +  C_1 P\eta),
\end{equation}
where $C_0$ and $C_1$ are constants independent of $P$.
As a comparison, we rewrite the bound of PRLC in (\ref{theorem1:bound1}) as
\begin{equation}
    \label{eq:SimpleBoundPRLC}
    \frac{1}{T} \sum\limits_{t=1}^{T} \mathbb{E} \|\nabla F(\omega_t)\|_2^2 
    \leq \frac{2(F(\omega_1) - F(\omega^*))}{T\eta} + \eta\Big[\frac{C_0'}{P} + \eta(\frac{C_1'}{P} + C_2')\Big],
\end{equation}
where $C_0', C_1'$, and $C_2'$ are constants independent of $P$. In bound 
(\ref{eq:ASGD}) of PR, the non-zero term grows with the scalability $P$, 
while the non-zero term (\ref{eq:SimpleBoundPRLC}) in bound of PRLC decays 
with scalability $P$. Hence, PRLC scales better than ASGD.
\end{proof}

\subsection{Extensions: Convergence in Convex Setting}
\begin{theorem}\label{Theorem:T8}
    (Convergence property, $c$-Strongly convex objective) When algorithm is 
    running with fixed learning rate $\eta$ for all iterations, and the learning rate for every iteration satisfies:
    \begin{equation}
    \label{equ:etaStrong}
    0 < \eta \le \frac{-r^2 + 2r\sqrt{r^2 + 16P(1-r)(2-r)}}{8PL(1-r)(2-r)},
    \end{equation}
    then the expected optimality gap satisfies the following inequality for all $t\in \mathbb{N}$:
    \begin{align}
        \label{theorem8:bound1}
        &\mathbb{E} \left[F(\omega_{t+1} - F(\omega^*)\right] \le \frac{D_{\eta,L,r,G,\sigma,P}}{\eta c} \nonumber \\
        &+ (1-\eta c)^t\left[F(\omega_1) - F(\omega^*) - \frac{D_{\eta,L,r,G,\sigma,P}}{\eta c}\right],
    \end{align}
    where $D_{\eta,L,r,G,\sigma,P} = \frac{2\eta^2 L^2 (1-r)(2-r)(\eta P^2 G^2 + 2\sigma^2) + 2 \eta^2 \sigma^2 r^2}{Pr^2} $.
\end{theorem}
\begin{proof}
    See Appendix of full paper.
\end{proof}

The Theorem~\ref{Theorem:T8} clearly shows that PRLC achieves a linear convergence rate for the strongly-convex 
objective. The convergence result is the same as sequential SGD (See Theorem 4.6 in \cite{DBLP:journals/siamrev/BottouCN18}).

\section{Evaluations}\label{Section:Evaluations}

\subsection{Experiment Settings}\label{label:ExpSetting}

We first measure the convergence and the communication rounds reduction of our method compared to Naive SGD (NSGD) and LAG-PS \cite{chen2018lag}.
The communication rounds are compared by using the average of accumulated pulling rounds of each worker. We then investigate the impact of different pulling ratios. Next, we present the efficiency of local compensation by comparing PRLC to PR. Finally, the time performance of PRLC is tested in a simulated setting of the edge environment.

\begin{figure}[htbp]
    \subfigure[LR]{
    \label{fig1:CommReductionLR}
    \begin{minipage}[t]{\columnwidth}
        \includegraphics[width=0.49\columnwidth, height=0.356\columnwidth]{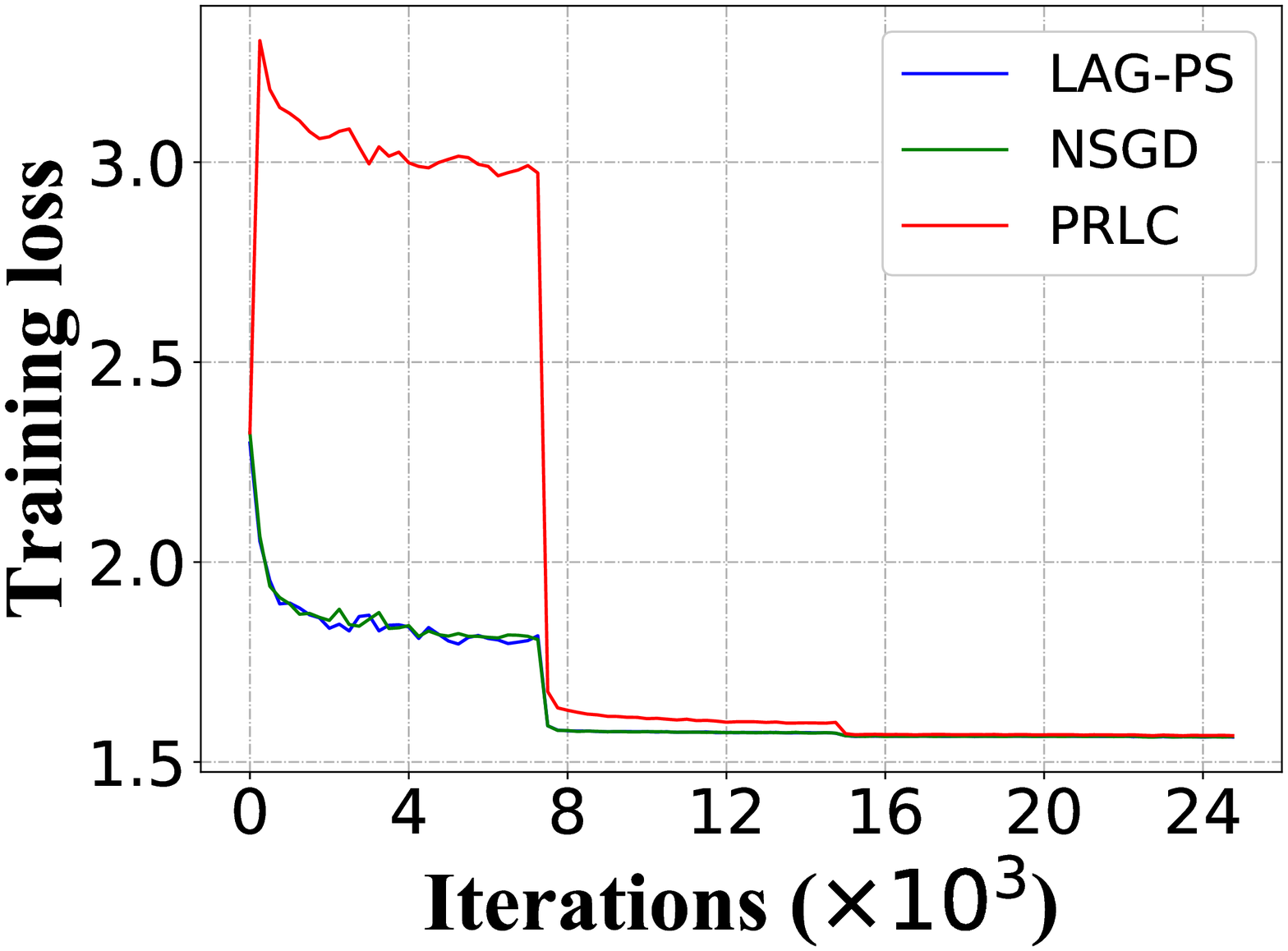} 
        \includegraphics[width=0.49\columnwidth, height=0.356\columnwidth]{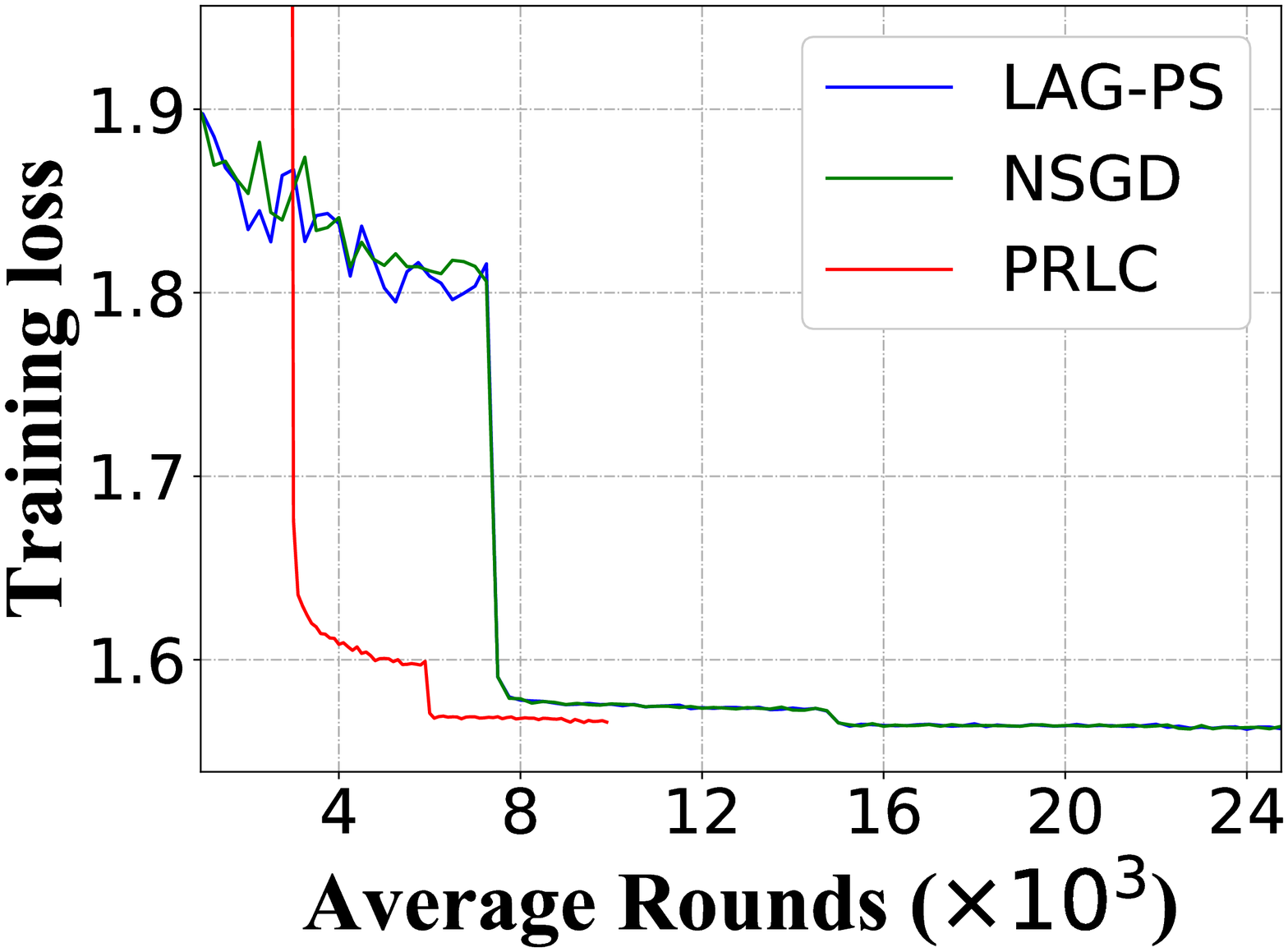}
    \end{minipage}
    }
    \subfigure[ResNet18]{
    \label{fig2:CommReductionResNet18}
    \begin{minipage}[t]{\columnwidth}
        \includegraphics[width=0.49\columnwidth, height=0.356\columnwidth]{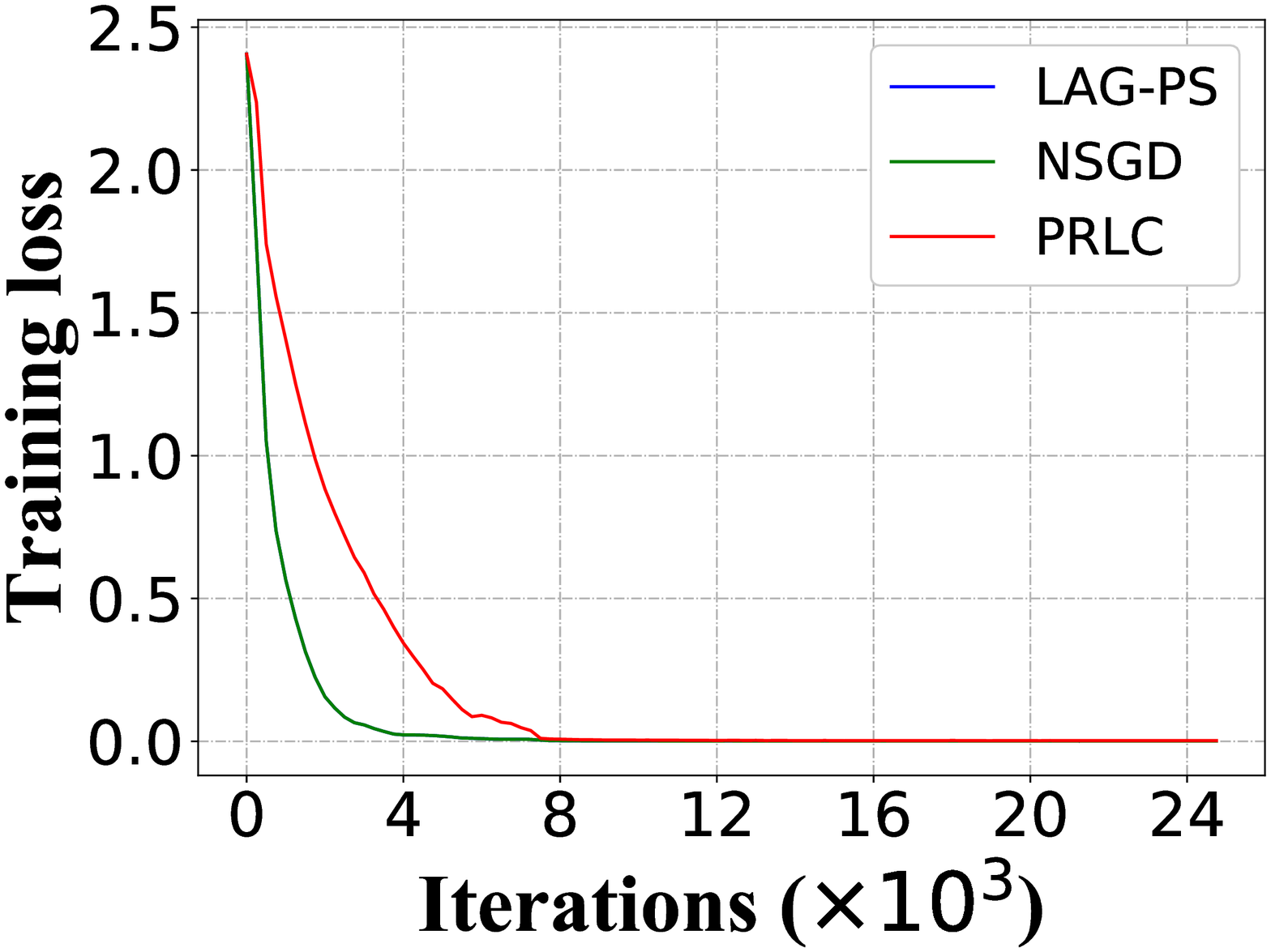}
        \includegraphics[width=0.49\columnwidth, height=0.356\columnwidth]{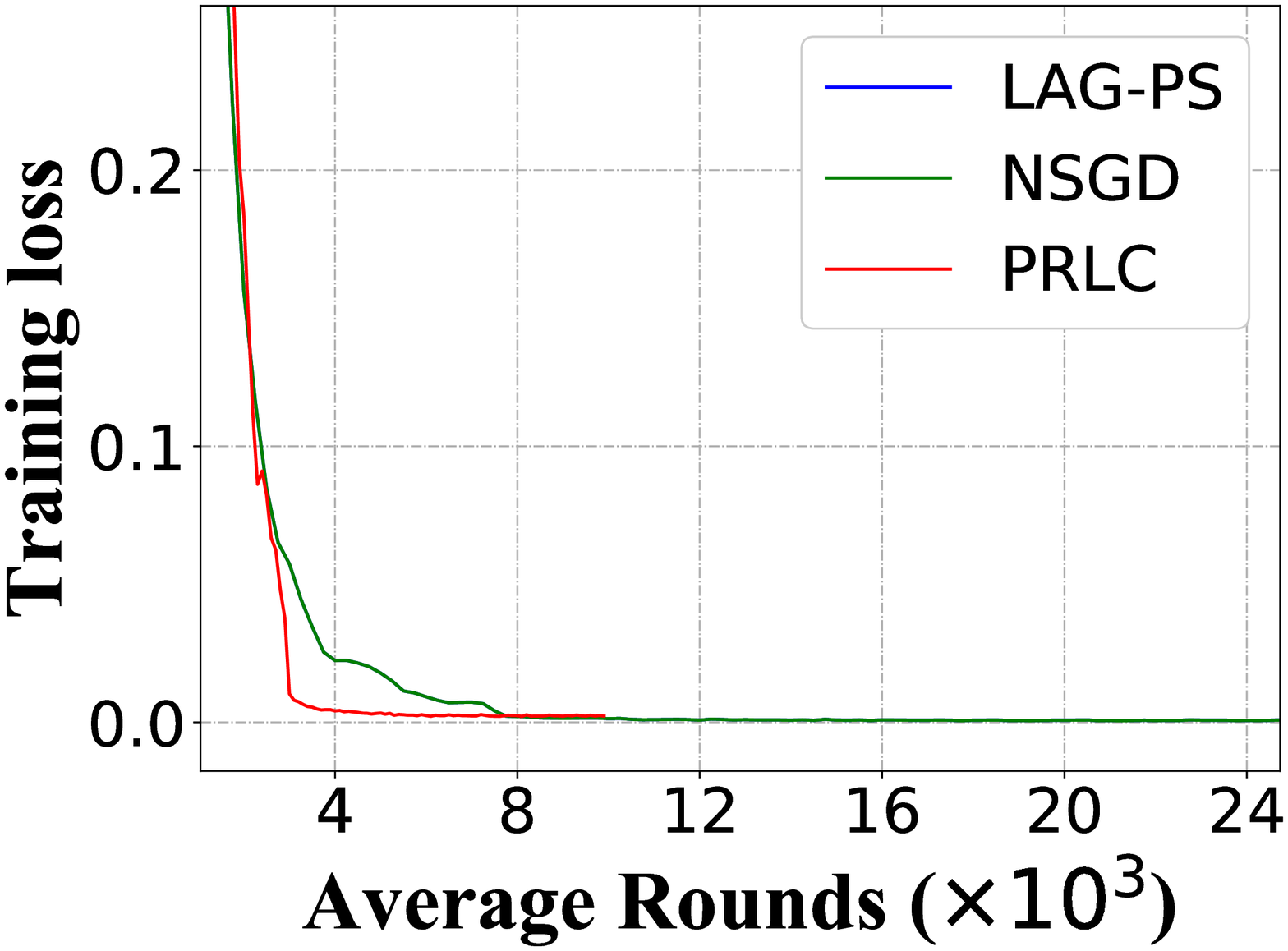}
    \end{minipage} 
    }
    \caption{Comparison with baseline. Experiments on 20-workers setting. 
    The hyper-parameters of LAG-PS are configured as the suggestion by the original paper. 
    The pulling ratio of PRLC is $r=0.4$.}
    \label{fig:CommReduction}
\end{figure}

We use GTX 1080TI GPUs to test our experiments with each GPU as a node. 
To simulate the real scenario of Federated Learning in an edge environment, 
we also conduct experiments on a cluster with low computing capacity and bandwidth. 
The cluster contains 20 virtual machines (VMs) as workers with each 
VM being configured $4$ CPU cores (2.6GHz) and 6GB RAM, and a VM as master 
with $16$ CPU cores (2.6GHz) and 16GB RAM.
We configure the bandwidth of the master VM to be 100MB/s via Linux Traffic 
Control tool \cite{LTC}. We implement all models and experiments 
on PyTorch 1.0 \cite{pytorch}. Our source code will be open after paper being accepted. 

We evaluate our method on CIFAR-10 \cite{krizhevsky2009learning} dataset 
using Logistic Regression (LR) and ResNet18 \cite{he2016deep} which cover 
both convex and non-convex models. For all experiments, the batch size is 
set to be $10$ in each worker, and the initial learning rate is set to 
be $0.1$ and decays by multiplying $0.1$ every 30 epochs.

\subsection{Results}

\subsubsection{Communication Rounds Reduction}\label{subsection:baseline}
The results of convergence and pulling communication rounds reduction of all methods are shown in Figure~\ref{fig:CommReduction}. The results show that PRLC has lower convergence rate in the first stages than NSGD. But, PRLC degrades rapidly after the learning rate being decayed and finally converges to the same loss floor which corresponds to the analysis of Theorem~\ref{Theorem:T1}. 
As to the pulling communication rounds reduction, it clearly shows that 
PRLC achieves a better result than 
LAG-PS which has no reduction at all. This is because reusing gradient of 
LAG-PS does not hold in the stochastic cases where the gradients vary largely.
Compared to NSGD, PRLC requires 
approximate $0.5$ pulling communication rounds when reaching convergence 
for both LR and ResNet18 models.

\begin{figure}[htbp]
    \subfigure[LR]{
    \label{fig1:ImpactRatioLR}
    \begin{minipage}[t]{0.46\columnwidth}
        \includegraphics[width=1.1\columnwidth, height=0.8\columnwidth]{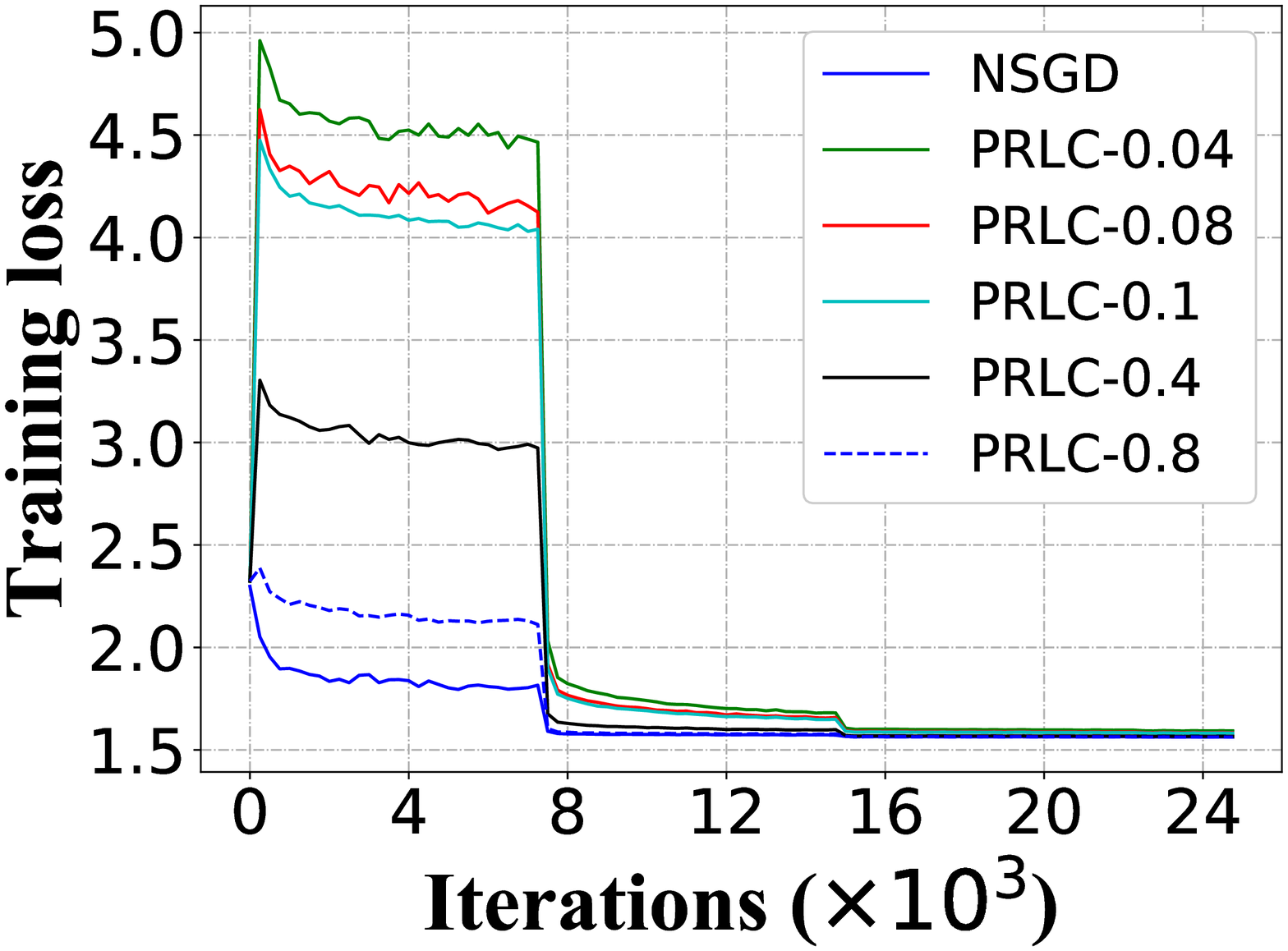} 
    \end{minipage}
    }
    \subfigure[ResNet18]{
    \label{fig2:ImpactRatioResNet18}
    \begin{minipage}[t]{0.46\columnwidth}
        \includegraphics[width=1.1\columnwidth, height=0.8\columnwidth]{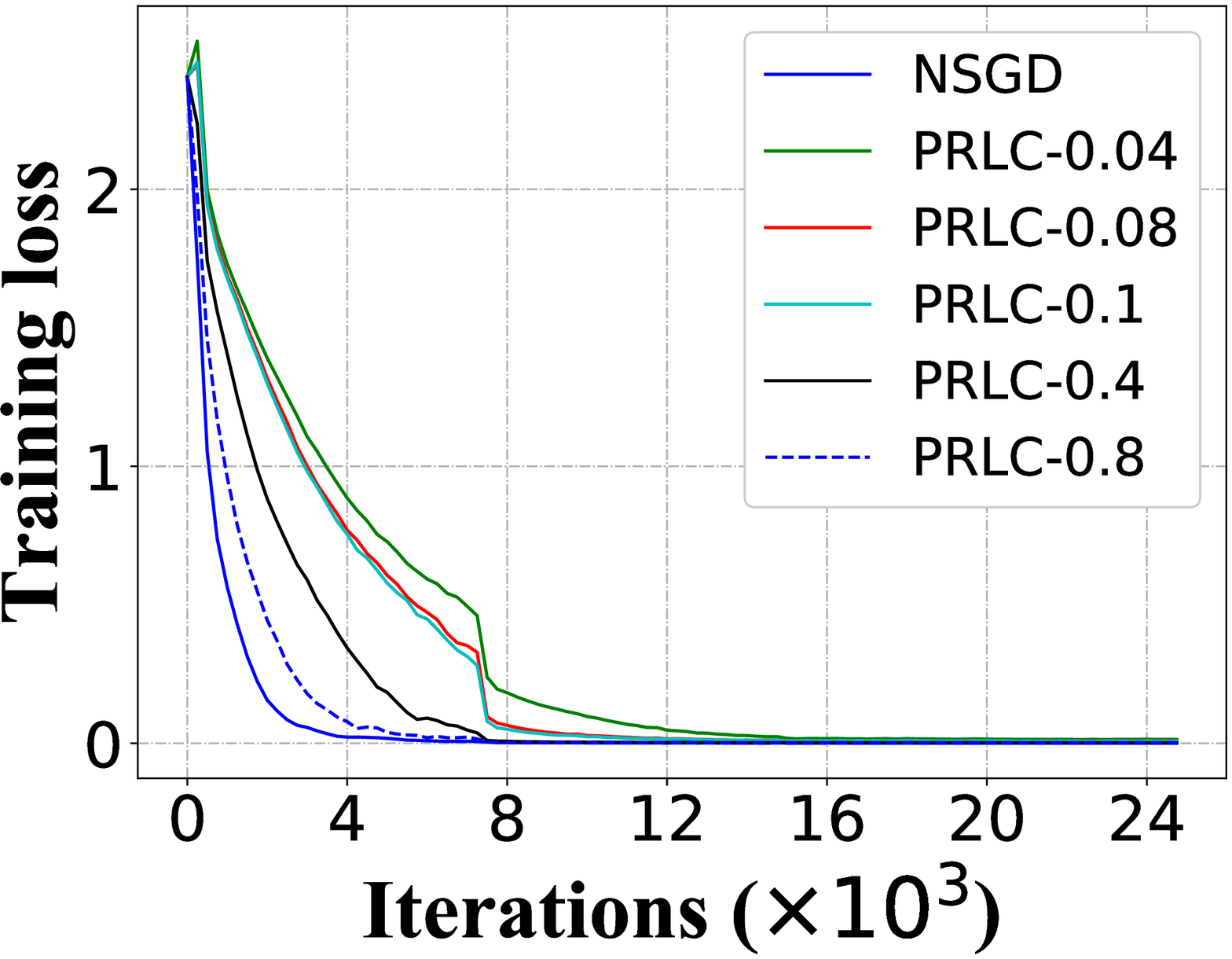}
    \end{minipage}
    }
    \caption{Impact of pulling ratio. PRLC-$x$ in the legend denotes that $x$ pulling ratio is adopted.}
    \label{fig:ImpactRatio}
\end{figure}

\subsubsection{Impact of Pulling Ratio}\label{subsection:PullRatio}
The impact of different pulling ratios is shown in Figure~\ref{fig:ImpactRatio}.
The term $A_{\eta,L,r,G,\sigma,P}$ in Theorem~\ref{Theorem:T1} for PRLC shows that the convergence gap of PRLC grows as the pulling ratio increases but could be reduced in a second-order speed by decaying learning rate. The results verify the analysis, in which the convergence loss gap between different ratios is large in the first stage but reduces significantly as the learning rate decays. 

\begin{figure}[htbp]
    \subfigure[LR]{
    \label{fig1:ImpactCompensationLR}
    \begin{minipage}[t]{0.46\columnwidth}
        \includegraphics[width=1.1\columnwidth, height=0.8\columnwidth]{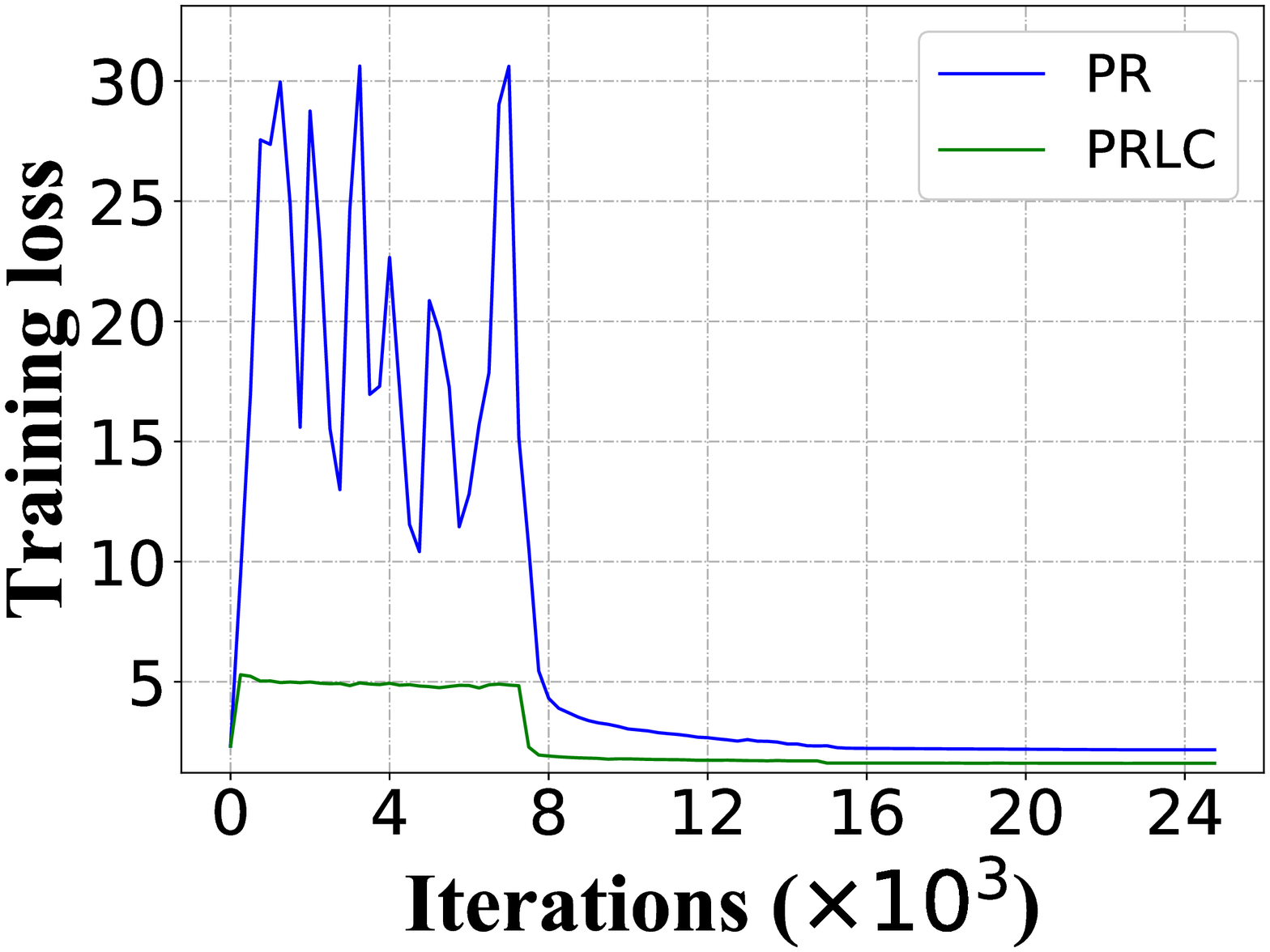} 
    \end{minipage}
    }
    \subfigure[ResNet18]{
    \label{fig2:ImpactCompensationResNet18}
    \begin{minipage}[t]{0.46\columnwidth}
        \includegraphics[width=1.1\columnwidth, height=0.8\columnwidth]{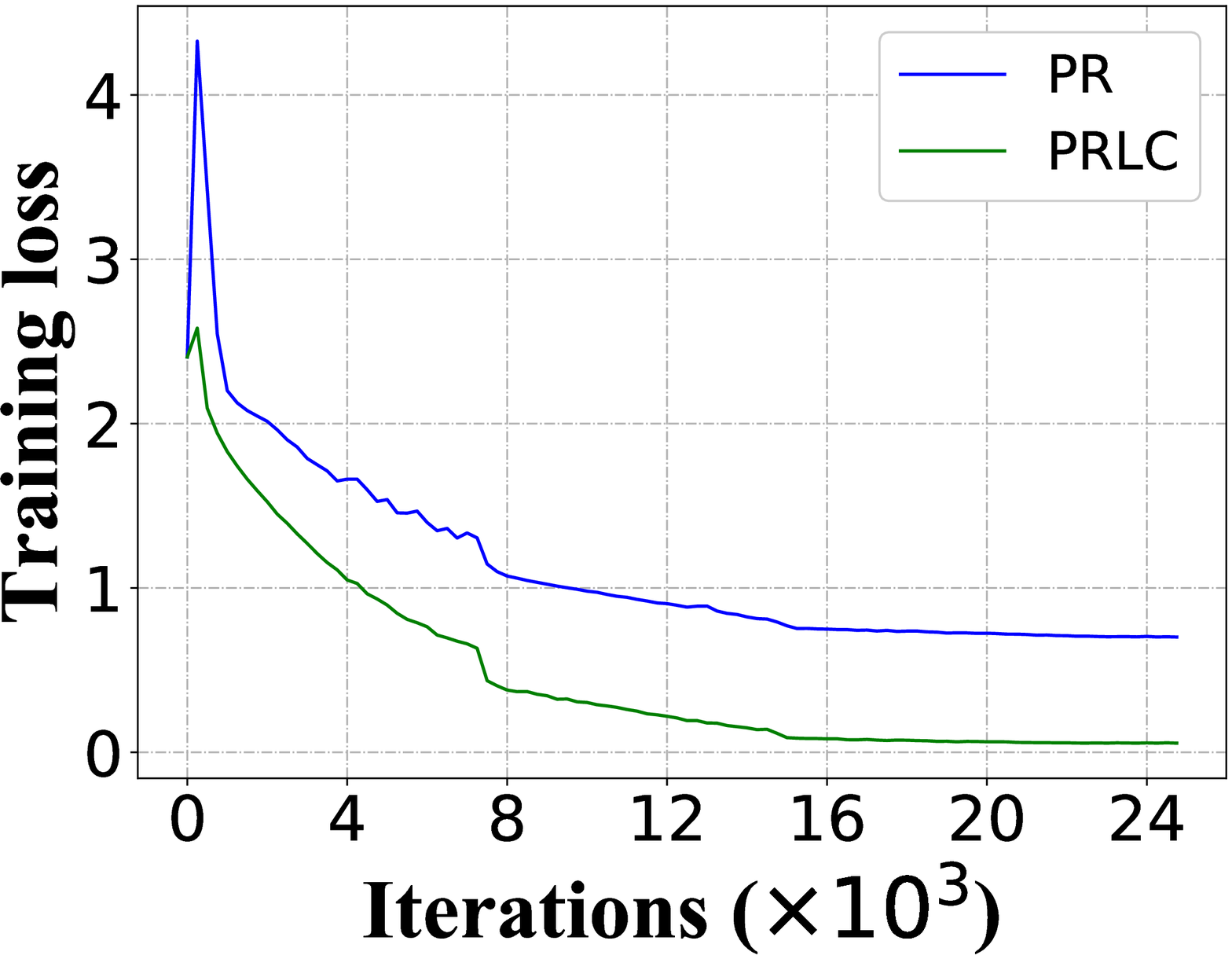}
    \end{minipage}
    }
    \caption{Impact of Compensation. The number of workers in the experiment are 20.
             PR and PRLC are compared in a pulling ratio of $r=0.01$.}
    \label{fig:ImpactCompensation}
\end{figure}

\subsubsection{Impact of Compensation}\label{subsection:Compensation}
Figure~\ref{fig:ImpactCompensation} shows the impact of compensation. 
For both convex and non-convex models, PRLC outperforms PR significantly. 
This is because PRLC compensates the update gap in the intermittent iterations 
with local update while PR not. Consequently, the local compensation of PRLC 
de-facto accelerates the convergence of PRLC by compensating the gap of 
intermittent pulling of the workers. 

\subsubsection{Performance}\label{subsection:Performance}
Finally, we compare PRLC to ASGD in terms of the time. The 
experiments are done in a simulated edge environment described in 
subsection~\ref{label:ExpSetting}.
The results are shown in Figure~\ref{fig:TimePerformance}. 
PRLC performs nearly $30\%$ better than ASGD because approximate 
$10\times10^4 s$ and $14 \times 10^4 s$ are required by 
PRLC and ASGD respectively when reaching convergence. Predictably, the improvement of PRLC over ASGD would become larger in the environment with lower bandwidth. 

\begin{figure}[htbp]
    \centering
    \includegraphics[width=0.9\columnwidth, height=0.486\columnwidth]{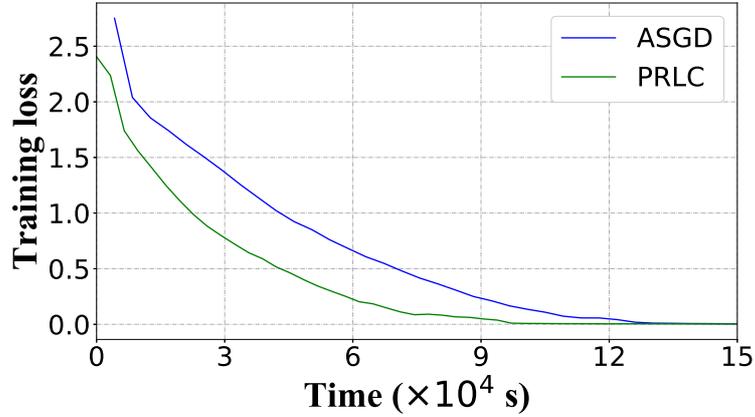}
    \caption{Performance comparison with ASGD. 
    The pulling ratio of PRLC is set to be $0.4$.}
    \label{fig:TimePerformance}
\end{figure}

\section{Conclusion and Future Work}\label{Section:Conclusion}
In this paper, we propose a novel method for reducing pulling communication rounds, called PRLC. In PRLC, each worker intermittently pulls the global model from the server and compensates the gap with its local update when not pulling the model. We establish the convergence theory of PRLC for both strongly convex and general non-convex cases. 
The theoretical results show that the convergence rate of PRLC is in the same order as the non-reduction method. Besides, we also prove that our method has a better 
tolerance for low pulling ratio over the reduction method without local compensation 
and has better scalability than ASGD.
To validate the efficiency of our method, extensive experiments are conducted for both strongly-convex and non-convex machine learning models by using general datasets. 
Experimental results present that PRLC significantly improves the efficiency in 
terms of the reduction of both the communication rounds and the convergence time.

    In the future, we try to consider more factors of the Federated Learning scenarios. 
For example, Non-IID of the datasets between workers is one of the main property in the edge environment. However, our method has not taken this property into account which still has a lot of room for improvement.

\clearpage    



\bibliographystyle{ieeetr}
\bibliography{reference}

\newpage
\appendix
\section{General Lemmas}
In this section, we present the lemma that is general for proofs of all theorems.
\begin{lemma}\label{LM:BoundGradient}
Let $\xi$ be the set of mini-batches, we have $\mathbb{E}_{\xi} [\| \sum_{i = 1}^P g(\omega_t^i; \xi_t^i) \|^2] \le 2P\sigma^2 + 2\sum_{i=1}^P \| \nabla F(\omega_t^i)  \| ^2$ for any iteration $t$.
\end{lemma}
\begin{proof}
By taking the expectation of the stochastic gradient with respect to $\xi$, we get that
\begin{align}
{} & \mathbb{E}_{\xi} [\| \sum_{i = 1}^P g(\omega_t^i; \xi_t^i) \|^2] \nonumber\\
= & \mathbb{E}_{\xi} [\| \sum_{i = 1}^P  g(\omega_t^i; \xi_t^i) - \sum_{i = 1}^P  \nabla F(\omega_t^i) + \sum_{i = 1}^P  \nabla F(\omega_t^i)  \| ^2] \nonumber\\
\le & 2 \mathbb{E}_{\xi} \| \sum_{i = 1}^P  [g(\omega_t^i; \xi_t^i) -  \nabla F(\omega_t^i)] \|^2 + 2 \mathbb{E}_{\xi} \| \sum_{i = 1}^P  \nabla F(\omega_t^i)  \| ^2 \nonumber\\
\overset{(a)}{\le} & 2 \sum_{i = 1}^P  \mathbb{E}_{\xi} [\| g(\omega_t^i; \xi_t^i) - \nabla F(\omega_t^i) \|^2] + 2 \mathbb{E}_{\xi} \| \sum_{i = 1}^P  \nabla F(\omega_t^i)  \| ^2 \nonumber\\
\overset{(b)}{\le} & 2P \sigma^2 + 2 \mathbb{E}_{\xi} \| \sum_{i = 1}^P  \nabla F(\omega_t^i)  \| ^2
\end{align}
where (a) follows according to that $\xi_t^i$ are i.i.d. and the summation rule of variance for independent variables ($\texttt{var}(a+b) = \texttt{var}(a) + \texttt{var}(b)$), and (b) comes after the assumption of bounded variance.
\end{proof}


\section{Proofs for PRLC}

Different from sequential SGD mechanisms, in our algorithm the global parameter and local parameter are not exactly the same for each worker, since each worker randomly pulls the global parameter and updates its own model. Therefore, we should derive the bound of the differential between global parameter and local parameter.

For any worker $i$ and iteration $t$, if worker $i$ last updated the global parameter in iteration $t-k$, where $k = 0,1,\ldots, t-1$, i.e., $\omega_{t-k}^i = \omega_{t-k}$ then we have:
\begin{equation}
\label{equ:localP}
\omega_{t}^i = \omega_{t-k} - \sum_{j=1}^k \eta g(\omega_{t-j}^i; \xi_{t-j}^i)
\end{equation}
and
\begin{equation}
\label{equ:globalP}
\omega_t = \omega_{t-k} - \frac{\eta}{p}\sum_{j=1}^k \sum_{i=1}^P \eta g(\omega_{t-j}^i; \xi_{t-j}^i)
\end{equation}
Therefore, we get that:
\begin{align}
\label{equ:parameterGap}
{} & \| \omega_t - \omega_t^i \|^2 \nonumber\\
 = & \| \sum_{j=1}^k \eta g(\omega_{t-j}^i; \xi_{t-j}^i) - \frac{\eta}{P}\sum_{j=1}^k \sum_{i=1}^P \eta g(\omega_{t-j}^i; \xi_{t-j}^i) \|^2 \nonumber\\
 \le & 2 \eta^2 \|\sum_{j=1}^k g(\omega_{t-j}^i; \xi_{t-j}^i)\|^2 + \frac{2\eta^2}{P^2} \|\sum_{j=1}^k \sum_{i = 1}^P g(\omega_{t-j}^i; \xi_{t-j}^i)\|^2 \nonumber\\
\end{align}

The expectation of the gap between the gradient with respect to global parameter and that with respect to local parameter has the following property:
\begin{lemma} \label{LM:GradientDifference}
For any worker $i$ and iteration $t$, we have: 
\begin{align}
{} & \mathbb{E}_\xi \left[ \| \nabla F(\omega_t) - \nabla F(\omega_t^i) \|^2 \right] \le \frac{4\eta^2 L^2 G^2 (1-r)(2-r)}{r^2} \nonumber\\
+ & \frac{4\eta^2 L^2 r}{P^2} \sum_{\ell = 0}^{t-1} \sum_{m=1}^\ell [(1-r)^{t-m}(t-m) \mathbb{E}_\xi \| \sum_{i=1}^P g(\omega_\ell^i; \xi_\ell^i) \|^2 ] \nonumber
\end{align}
\end{lemma}
\begin{proof}
With the $L$-smooth assumption and (7), we have:
\begin{align}
\label{eq:tempq}
{} & \mathbb{E}_{\xi} \left[ \| \nabla F(\omega_t) - \nabla F(\omega_t^i) \|^2 \right]
\le 2L^2 \mathbb{E}_{\xi} \left[ \| \omega_t - \omega_t^i \|^2 \right]\ \nonumber\\
\le & 4\eta^2 L^2  \mathbb{E}_{k,\xi} \|\sum_{j=1}^k g(\omega_{t-j}^i; \xi_{t-j}^i)\|^2 
 + \frac{4\eta^2 L^2}{P^2} \mathbb{E}_{k,\xi} \|\sum_{j=1}^k \sum_{i = 1}^P g(\omega_{t-j}^i; \xi_{t-j}^i)\|^2
\end{align}
Note that $k$ is a random variable, and for all $\ell = 0, 1, \ldots, t-1$, we have $\mathbb{P}[k = \ell] = r(1-r)^\ell$, where $r$ is the probability of pulling global parameter in each iteration. With the assumption of bounded gradient, we have:
\begin{align}
\label{equ:expectation1}
{} & \mathbb{E}_{\{k,\xi\}} \|\sum_{j=1}^k g(\omega_{t-j}^i; \xi_{t-j}^i)\|^2 \nonumber\\
=  & \sum_{\ell = 0}^{t-1} [ \mathbb{P}[k = \ell] \mathbb{E}_{\xi} \|\sum_{j=1}^\ell g(\omega_{t-j}^i; \xi_{t-j}^i)\|^2 ] \nonumber\\
=  & \sum_{\ell = 0}^{t-1} [ r(1-r)^\ell \mathbb{E}_{\xi}  \|\sum_{j=1}^\ell g(\omega_{t-j}^i; \xi_{t-j}^i)\|^2 ] \nonumber\\
\le  & r \sum_{\ell = 0}^{t-1} [ (1-r)^\ell \ell \sum_{j=1}^\ell \mathbb{E}_{\xi}  \|g(\omega_{t-j}^i; \xi_{t-j}^i)\|^2 ] \nonumber\\
\overset{(a)}{\le} & r G^2 \sum_{\ell = 0}^{t-1} (1-r)^\ell \ell^2 \overset{(b)}{\le} \frac{(1-r)(2-r)G^2}{r^2}
\end{align}
where (a) comes after bounded gradient assumption, and (b) follows according to that $\lim_{t \rightarrow \infty}\sum_{\ell = 0}^{t-1} (1-r)^\ell \ell^2 = \frac{(1-r)(2-r)}{r^3}$.

Similar to (\ref{equ:expectation1}),
\begin{align}
\label{equ:expectation2}
{} & \mathbb{E}_{\{k,\xi\}} \|\sum_{j=1}^k \sum_{i = 1}^P g(\omega_{t-j}^i; \xi_{t-j}^i)\|^2 \nonumber\\
\le  & r \sum_{\ell = 0}^{t-1} [ (1-r)^\ell \ell \mathbb{E}_{\xi} \sum_{j=1}^\ell \|\sum_{i = 1}^P g(\omega_{t-j}^i; \xi_{t-j}^i)\|^2 ] \nonumber\\
= & r \sum_{\ell = 0}^{t-1} \sum_{m=1}^\ell [(1-r)^{t-m}(t-m) \mathbb{E}_{\xi} \| \sum_{i=1}^P g(\omega_\ell^i; \xi_\ell^i) \|^2 ]
\end{align}
Considering that variable $j$ is constrained by variable $\ell$, the coefficient of $\mathbb{E}_{\xi} \| \sum_{i=1}^P g(\omega_\ell^i; \xi_\ell^i) \|$ is $r \sum_{m=1}^\ell (1-r)^{t-m}(t-m)$ for any $\ell \le t$. Therefore, the last step of above equation holds.

Replacing $\mathbb{E}_{\{k,\xi\}} \|\sum_{j=1}^k g(\omega_{t-j}^i; \xi_{t-j}^i)\|^2$ and $\mathbb{E}_{\{k,\xi\}} \|\sum_{j=1}^k \sum_{i = 1}^P g(\omega_{t-j}^i; \xi_{t-j}^i)\|^2$ in (\ref{eq:tempq}) immediately yields the result. 
\end{proof}

\begin{lemma}
Given any iteration $T$, the bound of the summation $\sum_{t = 1}^T \mathbb{E} [ \| \nabla F(\omega_t) - \nabla F(\omega_t^i) \|^2 ]$ can be formulated by location gradient.
\end{lemma}
\begin{proof}
\begin{align}
{} & \sum_{t=1}^T \sum_{\ell = 0}^{t-1} \sum_{m=1}^\ell [(1-r)^{t-m}(t-m) \mathbb{E}_\xi \| \sum_{i=1}^P g(\omega_\ell^i; \xi_\ell^i) \|^2] \nonumber\\
\le & \sum_{t=1}^{T-1} \sum_{m=1}^{T-1}[(1-r)^{t-m}(t-m)^2 \mathbb{E}_\xi \| \sum_{i=1}^P g(\omega_\ell^i; \xi_\ell^i) \|^2] \nonumber\\
\le & \frac{(1-r)(2-r)}{r^3}\sum_{t=1}^{T-1} \mathbb{E}_\xi \| \sum_{i=1}^P g(\omega_\ell^i; \xi_\ell^i) \|^2
\end{align}
where the first inequality follows according to that the part $(1-r)^{t-m}(t-m) \mathbb{E}_\xi \| \sum_{i=1}^P g(\omega_\ell^i; \xi_\ell^i) \|^2$ occurs at most $t-m$ times, considering the constraints of $m \le \ell$ and $\ell < t$.

With the result of Lemma~\ref{LM:GradientDifference}, we have:
\begin{align}
\label{equ:lemma4}
{} & \sum_{t=1}^T \mathbb{E}_\xi \left[ \| \nabla F(\omega_t) - \nabla F(\omega_t^i) \|^2 \right] \nonumber\\
\le & \frac{4\eta^2 L^2 G^2 T (1-r)(2-r)}{r^2} 
+ \frac{4\eta^2 L^2 (1-r)(2-r)}{P^2 r^2} \sum_{t=1}^{T-1} \mathbb{E}_\xi \| \sum_{i=1}^P g(\omega_\ell^i; \xi_\ell^i) \|^2
\end{align}
\end{proof}

\subsection{Convergence proof for Theorem 1}
\begin{proof}
Since function $F$ is $L$-smooth, for any $t>0$ we get that
\begin{align}
{} & F(\omega_{t+1}) - F(\omega_t) \nonumber\\
\le & \nabla F(\omega_t)(\omega_{t+1}-\omega_t)^T + \frac{L}{2}\|\omega_{t+1}-\omega_t\|^2 \nonumber\\
= & \nabla F(\omega_t)(-\frac{\eta}{P}\sum_{i=1}^P g(\omega_t^i;\xi_t^i))^T + \frac{L}{2}\|\frac{\eta}{P}\sum_{i=1}^P g(\omega_t^i;\xi_t^i)\|^2 \nonumber\\
\end{align}
Taking the expectation for both side with respect to $\xi$, we have
\begin{align}
{} & \mathbb{E} [F(\omega_{t+1}) - F(\omega_t)] \nonumber\\
= & -\frac{\eta}{2}\mathbb{E}\|\nabla F(\omega_t)\|^2 - \frac{\eta}{2P^2}\mathbb{E}\|\sum_{i=1}^P \nabla F(\omega_t^i)\|^2 \nonumber\\
+ &  \frac{\eta}{2}\mathbb{E}\|\nabla F(\omega_t) - \frac{1}{P}\sum_{i=1}^P \nabla F(\omega_t^i)\|^2 + \frac{L\eta^2}{2P^2}\mathbb{E}\|\sum_{i=1}^P g(\omega_t^i;\xi_t^i)\|^2 \nonumber \\
= & -\frac{\eta}{2}\mathbb{E}\|\nabla F(\omega_t)\|^2 + \frac{2 L\eta^2 - \eta}{2P^2}\mathbb{E}\|\sum_{i=1}^P \nabla F(\omega_t^i)\|^2 + \frac{L\eta^2\sigma^2}{P} \nonumber\\
+ & \frac{\eta}{2P^2}\mathbb{E}\| \sum_{i=1}^P [\nabla F(\omega_t) - \nabla F(\omega_t^i)]\|^2 \nonumber \\
= & -\frac{\eta}{2}\mathbb{E}\|\nabla F(\omega_t)\|^2 + \frac{2 L\eta^2 - \eta}{2P^2}\mathbb{E}\|\sum_{i=1}^P \nabla F(\omega_t^i)\|^2 + \frac{L\eta^2\sigma^2}{P} \nonumber\\
+ & \frac{\eta}{2P}\sum_{i=1}^P \mathbb{E}\|  \nabla F(\omega_t) - \nabla F(\omega_t^i)\|^2 \label{bd:OneIterationBDOfF}
\end{align}
Summing up the above equation for $t=1$ to $T$ for both sides, we get that
\begin{align}
{} & \mathbb{E} F(\omega_{t+1}) - F(\omega_1) \nonumber\\
\le & -\frac{\eta}{2}\sum_{t=1}^T \mathbb{E} \|\nabla F(\omega_t)\|^2 + \frac{2 L\eta^2 - \eta}{2P^2}  \sum_{t=1}^T \mathbb{E} \|\sum_{i=1}^P \nabla F(\omega_t^i)\|^2  \nonumber\\
+ & \frac{\eta}{2P}\sum_{i=1}^P \sum_{t=1}^T \mathbb{E} \| \nabla F(\omega_t) - \nabla F(\omega_t^i)\|^2 + \frac{L\eta^2\sigma^2T}{P} \nonumber\\
\le & -\frac{\eta}{2}\sum_{t=1}^T \mathbb{E} \|\nabla F(\omega_t)\|^2 + \frac{2 L\eta^2 - \eta}{2P^2} \mathbb{E} \sum_{t=1}^T \|\sum_{i=1}^P \nabla F(\omega_t^i)\|^2 \nonumber\\
+ & \frac{2\eta^3 L^2 G^2 T (1-r)(2-r)}{r^2} + \frac{L\eta^2\sigma^2T}{P} \nonumber\\
+ & \frac{2\eta^3 L^2 (1-r)(2-r)}{P^2 r^2} \sum_{t=1}^{T-1} \mathbb{E} \| \sum_{i=1}^P g(\omega_\ell^i; \xi_\ell^i) \|^2
\end{align}
where the last inequality follows according to (\ref{equ:lemma4}).

By replacing $\mathbb{E} \| \sum_{i=1}^P g(\omega_\ell^i; \xi_\ell^i) \|^2$ according to the result of Lemma~\ref{LM:BoundGradient}, we get that
\begin{align}
{} & \mathbb{E} F(\omega_{t+1}) - F(\omega_1) \nonumber\\
\le & -\frac{\eta}{2}\sum_{t=1}^T \|\nabla F(\omega_t)\|^2 + A_{\eta,L,r,G,\sigma} T \eta  \nonumber\\
{} & + B_{\eta,L,r,P} \mathbb{E} \sum_{t=1}^T \|\sum_{i=1}^P \nabla F(\omega_t^i)\|^2 \nonumber
\end{align}
where $A_{\eta,L,r,G,\sigma,P} = \frac{2\eta^2 L^2 (1-r)(2-r)(PG^2 + 2\sigma^2) + L \eta \sigma^2 r^2}{r^2P} $ and \\
$B_{\eta,L,r,P} = \frac{8\eta^3 L^2 (1-r)(2-r) + (2L\eta^2 - \eta) r^2}{2P^2 r^2}$ are two constants.

By setting $B_{\eta,L,r,P} < 0$, we can derive the satisfied stepsize, i.e., $0 < \eta \le \frac{-2 L r^2 + \sqrt{4L^2r^4 + 32L^2r^2(1-r)(2-r)}}{16L^2(1-r)(2-r)}$. Then we have
\begin{equation}
\mathbb{E} F(\omega_{t+1}) - F(\omega_1) \le  -\frac{\eta}{2}\sum_{t=1}^T \|\nabla F(\omega_t)\|^2 + A_{\eta,L,r,G,\sigma,P} T \eta
\end{equation}

Since $F(\omega_{T+1})$ is bound by $F(\omega^*)$, we have
\begin{equation}
\label{equ:convergence}
\frac{1}{T}\sum_{t=1}^T \mathbb{E} \|\nabla F(\omega_t)\|^2 \le \frac{2|F(\omega_{1}) - F(\omega^*)|}{\eta T} + 2A_{\eta,L,r,G,\sigma,P}
\end{equation}
\end{proof}

According to (\ref{equ:convergence}), the average norm of gradient converges to a non-zero constant $2A_{\eta,L,r,G,\sigma,P}$ under a fixed learning rate with $T\rightarrow \infty$. The constant $2A_{\eta,L,r,G,\sigma,P}$ is related with the stepsize $\eta$. By diminishing $\eta$ during the learning process, i.e., $\eta = O(1/\sqrt{T})$, it is easy to find that Algorithm 1 has a convergence rate of $O(1/\sqrt{T})$, as shown in the following theorem.

\subsection{Convergence proof for Theorem 2}
\begin{proof}
Considering the right side of (\ref{equ:convergence}) as a function of $\eta$, we have $f(\eta) = \frac{2[F(\omega_1)-F(\omega^*)]}{\eta T} +  2A_{\eta,L,r,G,\sigma,P}$. By ignoring the highest order of $\eta$ in $A_{\eta,L,r,G,\sigma,P}$, when
\begin{equation}
 \eta = \sqrt{\frac{[F\left(\omega_{1}\right)-F\left(\omega^*\right)]P }{L \sigma^{2} T}},
\end{equation}
we have
\begin{equation}
f(\eta) \leq 4 \sqrt{\frac{[F\left(\omega_{1}\right)-F\left(\omega^*\right)] \sigma^{2} L}{P}} * \frac{1}{\sqrt{T}} + O(\frac{1}{T}).
\end{equation}

Combining with the constraint of stepsize in (\ref{equ:eta}), we can derive the condition of $T$ in (\ref{equ:iteration}), which completes the proof.
\end{proof}

\subsection{Convergence proof for Theorem 5}
\begin{proof}
Base on inequality (\ref{bd:OneIterationBDOfF}), we have
\begin{align}
    &\mathbb{E}[F(\omega_{t+1})-F(\omega_t)]   \nonumber \\
    &\le - \frac{\eta }{2}\mathbb{E}{\left\| {\nabla F({\omega _t})} \right\|^2} + \frac{{2L{\eta ^2} - \eta }}{{2{P^2}}}\mathbb{E}{\left\| {\sum\limits_{i = 1}^P {\nabla F(\omega _t^i)} } \right\|^2} + \frac{{L{\eta ^2}{\sigma ^2}}}{P} \nonumber \\
    &+  \frac{\eta }{{2P}}\sum\limits_{i = 1}^P {\mathbb{E}{{\left\| {\nabla F({\omega _t}) - \nabla F(\omega _t^i)} \right\|}^2}} \nonumber
\end{align}
Since $2c(F(\omega) - F(\omega^*)) \le \parallel \nabla F(\omega) \parallel ^2$ in c-strongly convex function, the bound could be reformulated as
\begin{align}
    &\mathbb{E}[F(\omega_{t+1})-F(\omega_t)]   \nonumber \\
       &\le   - \eta c \mathbb{E}[F(\omega_t)-F(\omega^*)] + \frac{{2L{\eta ^2} - \eta }}{{2{P^2}}}\mathbb{E}{\left\| {\sum\limits_{i = 1}^P {\nabla F(\omega _t^i)} } \right\|^2} + \frac{{L{\eta ^2}{\sigma ^2}}}{P} \nonumber \\
       &+   \frac{\eta }{{2P}}\sum\limits_{i = 1}^P {\mathbb{E}{{\left\| {\nabla F({\omega _t}) - \nabla F(\omega _t^i)} \right\|}^2}}. \label{ieq:oneItBound}
\end{align}

Moving $\mathbb{E}F(\omega_t)$ to right side and $\mathbb{F(\omega^*)}$ to left side, the formula (\ref{ieq:oneItBound}) is transformed to

\begin{align}
    &\mathbb{E}[F(\omega_{t+1})-F(\omega^*)]   \\
       &\le   (1-\eta c)\mathbb{E}[F(\omega_t) - F(\omega^*)] + \frac{{2L{\eta ^2} - \eta }}{{2{P^2}}}\mathbb{E}{\left\| {\sum\limits_{i = 1}^P {\nabla F(\omega _t^i)} } \right\|^2} + \frac{{L{\eta ^2}{\sigma ^2}}}{P} \nonumber \\
       &+   \frac{\eta }{{2P}}\sum\limits_{i = 1}^P {\mathbb{E}{{\left\| {\nabla F({\omega _t}) - \nabla F(\omega _t^i)} \right\|}^2}} \nonumber \\
       &\le   (1-\eta c)\mathbb{E}[F(\omega_t) - F(\omega^*)] + \frac{{2L{\eta ^2} - \eta }}{{2{P^2}}}\mathbb{E}{\left\| {\sum\limits_{i = 1}^P {\nabla F(\omega _t^i)} } \right\|^2}  \nonumber \\
       &+ \frac{{L{\eta ^2}{\sigma ^2}}}{P} + \frac{2 \eta^3 L^2 G^2 (1-r)(2-r)}{r^2} \nonumber \\
       &+ \frac{1}{P^2}\left(2 \eta^3 L^2 r \sum\limits_{\ell=0}^{t-1}{\sum\limits_{m=1}^{\ell}{[(1-r)^{t-m} (t-m) \mathbb{E}\parallel \sum\limits_{i=1}^P {g(\omega_{\ell}^i, \xi_{\ell}^i)} \parallel^2 ]}}\right) \nonumber \\
       &\le   (1-\eta c)\mathbb{E}[F(\omega_t) - F(\omega^*)] + \frac{{2L{\eta ^2} - \eta }}{{2{P^2}}}\mathbb{E}{\left\| {\sum\limits_{i = 1}^P {\nabla F(\omega _t^i)} } \right\|^2} + \frac{{L{\eta ^2}{\sigma ^2}}}{P} \nonumber \\
       &+   \frac{2 \eta^3 L^2 G^2 (1-r)(2-r)}{r^2} + \frac{1}{\rm{P}}\left(4{\eta ^3}{L^2}r{\sigma ^2}\sum\limits_{\ell  = 0}^{t - 1} {\sum\limits_{m = 1}^\ell  {{{(1 - r)}^{t - m}}(t - m)} } \right) \nonumber \\
       &+   \frac{1}{P}\left(4{\eta ^3}{L^2}r\sum\limits_{\ell  = 0}^{t - 1} {\sum\limits_{m = 1}^\ell  {{{(1 - r)}^{t - m}}(t - m) \mathbb{E} {{\left\| {\sum\limits_{i = 1}^P {\nabla F(\omega _t^i)} } \right\|}^2}} } \right)\nonumber \\
       &\le   (1-\eta c)\mathbb{E}[F(\omega_t) - F(\omega^*)] + D + H \left\| {\sum\limits_{i = 1}^P {\nabla F(\omega _t^i)} } \right\| \label{ieq:recursiveItBound},
\end{align}
where $D$ and $H$ in (\ref{ieq:recursiveItBound}) are 
$$D= \frac{2 \eta^2 \sigma^2}{P} + \frac{2 \eta^3 L^2 G^2 (1-r) (2-r)}{r^2} + \frac{4 \eta^2 L^2 \sigma^2 (1-r)(2-r)}{P r^2}$$ 
and 
$$H=\frac{2L \eta^2 - \eta}{2P^2} + \frac{4 \eta^3 L^2 (1-r)(2-r)}{P r^2}$$
respectively. The last inequality is due to 
\begin{equation}\label{eq:sumEQ}
    \sum\limits_{\ell=0}^{t-1}{\sum\limits_{m=1}^{\ell}}{(1-t)^{t-m}(t-m)}  =  \sum\limits_{n=1}^{t-1} {(1-r)^n n^2} = \frac{(1-r)(2-r)}{r^3}, 
\end{equation}
as $t$ is large enough.

Setting $H \le 0$, the learning rate satisfies the following inequality
\begin{equation}
\label{bd:lrOfStrongConvex}
0 < \eta \le \frac{-r^2 + 2r \sqrt{r^2 + 16 P (1-r)(2-r)}}{8 P L (1-r)(2-r)}.
\end{equation}

The bound (\ref{ieq:recursiveItBound}) becomes 

\begin{align}
    \mathbb{E}[F(\omega_{t+1}) - F(\omega^*)]  &\le (1-\eta c) \mathbb{E}[F(\omega_{t}) - F(\omega^*)] + D \nonumber \\
    \mathbb{E}[F(\omega_{t+1}) - F(\omega^*)] - \frac{D}{\eta c}   &\le   (1-\eta c)(\mathbb{E}[F(\omega_{t}) - F(\omega^*)] - \frac{D}{\eta c}) \label{bd:recursiveOBD}
\end{align}
Iterating (\ref{bd:recursiveOBD}), Theorem~5 could be easily achieved.
\end{proof}

\section{Proofs for PR}
Before proving the convergence of PR, we present the required lemmas. 
Similar to PRLC, we also give the bound of the difference between local expected gradient and average gradient.

\begin{lemma} \label{LM:PRGradientDifference}
    For iterations from $t=1$ to $T$, the total difference between global gradient and local gradient is bounded by 
    \begin{align}
        &\sum_{t = 1}^T \mathbb{E} [ \| \frac{1}{P}\sum_{i=1}^P (\nabla F(\omega_t) - \nabla F(\omega_t^i)) \|^2 ] \nonumber \\
        &\le \sum_{t = 1}^T \frac{2\eta^2 L^2(1-r)(1-2^{T-t}(1-r)^{T-t})}{2r-1} \mathbb{E} \| \frac{1}{P}\sum_{i=1}^P g(\omega_t^i; \xi_t^i) \|^2.
    \end{align}
\end{lemma}
\begin{proof}
Based on Assumption 1, the gradient difference for any iteration $t$ is
\begin{equation}
    \label{eq:oneItGradientDifference}
    \mathbb{E} \| \frac{1}{P}\sum_{i=1}^P(\nabla F(\omega_t) - \nabla F(\omega_t^i)) \|^2 
    \le  \frac{L^2}{P}\sum_{i=1}^P \mathbb{E} \| \omega_t - \omega_t^i \|^2.
\end{equation}
Clearly, the key to bound (\ref{eq:oneItGradientDifference}) is to bound the term $\mathbb{E} \| \omega_t - \omega_t^i \|^2$. 
Without confusion, we use $\mathbb{E}$ to repsent taking expectation over both mini-batches $\xi_t^i$ and the random pulling of global model, 
then we have
\begin{align}
    &\mathbb{E} \| \omega_t - \omega_t^i \|^2 \nonumber \\
    &= r\mathbb{E} \| \omega_t - \omega_t \|^2 + (1-r) \mathbb{E} \| \omega_t - \omega_{t-1}^i \|^2 \nonumber \\
    &= (1-r)\mathbb{E} \| \omega_{t-1} - \omega_{t-1}^i - \frac{\eta}{P}\sum_{i=1}^P g(\omega_{t-1}^i; \xi_{t-1}^i) \|^2 \nonumber \\
    &\le 2(1-r) \mathbb{E} \| \omega_{t-1} - \omega_{t-1}^i \|^2 + 2(1-r)\mathbb{E}\|\frac{1}{P}\sum_{i=1}^P g(\omega_t^i; \xi_t^i)\|^2  \nonumber \\
    &\le \sum_{k=1}^{t-1}2^k(1-r)^k\mathbb{E}\|g(\omega_{t-k}^i; \xi_{t-k}^i)\|^2. \label{bd:oneItModelDifferencePR}
\end{align}

Based on (\ref{eq:oneItGradientDifference}) and (\ref{bd:oneItModelDifferencePR}), we could derive the following bound
\begin{align}
    &\sum_{t = 1}^T \mathbb{E} \| \frac{1}{P}\sum_{i=1}^P (\nabla F(\omega_t) - \nabla F(\omega_t^i)) \|^2 \nonumber \\
    &\le \frac{L^2}{P}\sum_{t=1}^{T}\sum_{i=1}^{P} \mathbb{E} \| \omega_t - \omega_t^i \|^2                 \nonumber \\
    &\le \frac{L^2}{P}\sum_{t=1}^{T}\sum_{j=1}^{P} \sum_{k=1}{t-1}2^k(1-r)^k\mathbb{E}\|\frac{1}{P}\sum_{i=1}^P g(\omega_{t-k}^i; \xi_{t-k}^i)\|^2  \nonumber \\
    &\le \frac{L^2}{P}\sum_{j=1}^{P}\sum_{t=1}^{T} \sum_{k=1}^{t-1}2^k(1-r)^k\mathbb{E}\|\frac{1}{P}\sum_{i=1}^P g(\omega_{t-k}^i; \xi_{t-k}^i)\|^2  \nonumber \\
    &= \sum_{t=1}^{T} \frac{2\eta^2 L^2(1-r)(1-2^{T-t}(1-r)^{T-t})}{2r-1} \mathbb{E} \| \frac{1}{P}\sum_{i=1}^P g(\omega_t^i; \xi_t^i) \|^2,
\end{align}
where the reason for last equality is similar to the inequality (\ref{eq:sumEQ}). 
Proof is done.
\end{proof}

\subsection{Convergence proof for Theorem 3}
\begin{proof}
Similar to the derivation of bound \ref{bd:OneIterationBDOfF}, we could achieve the following inequality based Assumption 1
\begin{align}
    {} & \mathbb{E} [F(\omega_{t+1}) - F(\omega_t)] \nonumber\\
    = & -\frac{\eta}{2}\mathbb{E}\|\nabla F(\omega_t)\|^2 - \frac{\eta}{2P^2}\mathbb{E}\|\sum_{i=1}^P \nabla F(\omega_t^i)\|^2 \nonumber\\
    + &  \frac{L\eta^2}{2P^2}\mathbb{E}\|\sum_{i=1}^P g(\omega_t^i;\xi_t^i)\|^2 + \frac{\eta}{2}\mathbb{E}\|\nabla F(\omega_t) - \frac{1}{P}\sum_{i=1}^P \nabla F(\omega_t^i)\|^2 \label{bd:lossOneItPR}
\end{align}

Summing (\ref{bd:lossOneItPR}) from $t=1$ to $T$, we have
\begin{align}
    {} & \mathbb{E} [F(\omega_{T+1}) - F(\omega_1)] \nonumber\\
    = & -\frac{\eta}{2}\sum_{t=1}^T \mathbb{E}\|\nabla F(\omega_t)\|^2 - \frac{\eta}{2P^2}\sum_{t=1}^T \mathbb{E}\|\sum_{i=1}^P \nabla F(\omega_t^i)\|^2 \nonumber\\
    + &  \frac{L\eta^2}{2P^2}\sum_{t=1}^T \mathbb{E}\|\sum_{i=1}^P g(\omega_t^i;\xi_t^i)\|^2 + \frac{\eta}{2}\sum_{t=1}^T \mathbb{E}\|\nabla F(\omega_t) - \frac{1}{P}\sum_{i=1}^P \nabla F(\omega_t^i)\|^2 \nonumber\\
    = & -\frac{\eta}{2}\sum_{t=1}^T \mathbb{E}\|\nabla F(\omega_t)\|^2 - \frac{\eta}{2P^2}\sum_{t=1}^T \mathbb{E}\|\sum_{i=1}^P \nabla F(\omega_t^i)\|^2 \nonumber\\
    + & \sum_{t=1}^T\left( \frac{\eta^3 L^2(1-r)(1-2^{T-t}(1-r)^{T-t})}{P^2(2r-1)} + \frac{L\eta^2}{2P^2}\right)\mathbb{E}\|\sum_{i=1}^P g(\omega_t^i;\xi_t^i)\|^2 \nonumber\\
    \le & -\frac{\eta}{2}\sum_{t=1}^T \mathbb{E}\|\nabla F(\omega_t)\|^2 + \sum_{t=1}^T P\sigma^2\left(\frac{2\eta^3 L^2(1-r)(1-2^{T-t}(1-r)^{T-t})}{P^2(2r-1)} + \frac{L\eta^2}{P^2} \right) \nonumber\\
    + & + \sum_{t=1}^{T}\left(\frac{2\eta^3 L^2(1-r)(1-2^{T-t}(1-r)^{T-t})}{P^2(2r-1)} + \frac{L\eta^2}{P^2} - \frac{\eta}{2P^2}\right)\mathbb{E}\|\sum_{i=1}^P \nabla F(\omega_t^i)\|^2 \nonumber \\
    \le & -\frac{\eta}{2}\sum_{t=1}^T \mathbb{E}\|\nabla F(\omega_t)\|^2 + \sum_{t=1}^{T}\left(\frac{2\eta^3 L^2(1-r)}{P^2(2r-1)} + \frac{L\eta^2}{P^2} - \frac{\eta}{2P^2}\right)\mathbb{E}\|\sum_{i=1}^P \nabla F(\omega_t^i)\|^2 \nonumber\\
    + & \frac{2\eta^3 L^2 \sigma^2}{P}\left[\frac{(1-r)T}{2r-1} - \frac{(1-r)(1-2^T(1-r)^T)}{(2r-1)^2T}\right] + \frac{L\eta^2\sigma^2T}{P}, \label{lb:bdAllIterationsPR}
\end{align}

Since $F(\omega^*) - F(\omega_1) \le \mathbb{E} [F(\omega_{T+1}) - F(\omega_1)]$, moving the accumulated square gradient 
to the left side of (\ref{lb:bdAllIterationsPR}) and dividing $\frac{\eta}{2}$ on both sides, we have

\begin{align}
    {}& \sum_{t=1}^T \mathbb{E}\|\nabla F(\omega_t)\|^2 \le \frac{2(F(\omega_{1}) - F(\omega^*)}{\eta T} +  \frac{2L\eta\sigma^2}{P}  \nonumber \\
    + & \frac{4\eta^2 L \sigma^2}{P}\left[\frac{1-r}{2r-1} - \frac{(1-r)(1-2^T(1-r)^T)}{(2r-1)^2T} \right]      \nonumber \\
    + & \sum_{t=1}^{T}H\mathbb{E}\|\sum_{i=1}^P \nabla F(\omega_t^i)\|^2, \label{bd:squareGradientPR}
\end{align}
where $H=\frac{2\eta^3 L^2(1-r)}{P^2(2r-1)} + \frac{L\eta^2}{P^2} - \frac{\eta}{2P^2}$. Considerring $H \le 0$, i.e.,
the learning rate satisfying 
$$0 < \eta \le \frac{-4 L + \sqrt{16L^2 + 32L(1-r)}}{16L(1-r)}$$, we could derive the following convergence result based on (\ref{bd:squareGradientPR})
\begin{align}
    &\frac{1}{T}\sum_{t=1}^T \mathbb{E} \|\nabla F(\omega_t)\|^2 \le \frac{2(F(\omega_{1}) - F(\omega^*)}{\eta T} +  \frac{2L\eta\sigma^2}{P}  \nonumber \\
    &+ \frac{4\eta^2 L \sigma^2}{P}\left[\frac{1-r}{2r-1} - \frac{(1-r)(1-2^T(1-r)^T)}{(2r-1)^2T} \right]
\end{align}
Obviously, $- \frac{(1-r)(1-2^T(1-r)^T)}{(2r-1)^2T}$ could be bounded as $0.5 \le r$, which completes the proof.
\end{proof}

\end{document}